\documentclass{IEEEtran}

\usepackage[utf8]{inputenc}

\usepackage{setspace}

\usepackage[colorlinks]{hyperref}

\usepackage{amsmath,amsfonts,amssymb,mathrsfs,amsthm}
\usepackage{bm}
\usepackage{microtype}
\usepackage{mathtools}
\usepackage{breqn}

\usepackage{upgreek}

\usepackage{textcomp,nicefrac}

\usepackage{algorithm}
\usepackage{algpseudocode}

\usepackage{glossaries-extra}

\usepackage{graphicx}
\usepackage{adjustbox}
\usepackage{caption}
\usepackage[labelformat=simple,font=footnotesize]{subcaption}
\usepackage[dvipsnames,table]{xcolor}
\usepackage{wrapfig}

\usepackage{tikz}
\usetikzlibrary{arrows.meta,shapes,intersections,calc,angles,decorations.pathreplacing,intersections,quotes,angles,positioning,fit,petri,through,shadows,plotmarks,spy}
\usepackage{pgfplots}
\pgfplotsset{compat=1.18}
\usepgfplotslibrary{groupplots}
\usepgfplotslibrary{fillbetween}

\usepackage{overpic}
\usepackage{stfloats}

\usepackage{booktabs}
\usepackage{multirow}

\usepackage[
	style=ieee,
	dashed=false,
	maxbibnames=20,
	minbibnames=2,
	maxcitenames=1,
	mincitenames=1,
	backend=biber,
	sortcites=true,
	sorting=none,
	useprefix=false
]{biblatex}

\newtheorem{proposition}{Proposition}
\newtheorem{assumption}{Assumption}

\newcommand{\ImageWithMethodLabels}[2][0.985\textwidth]{%
\begin{tikzpicture}
    \node[inner sep=0pt, anchor=north west] (img) at (0,0)
        {\includegraphics[width=#1]{#2}};

    \node[anchor=north, font=\bfseries\scriptsize, inner sep=0pt]
        at ($(img.south west)!0.0625!(img.south east)+(0,-2pt)$) {FBP};

    \node[anchor=north, font=\bfseries\scriptsize, inner sep=0pt]
        at ($(img.south west)!0.1875!(img.south east)+(0,-2pt)$) {DiffPIR};

    \node[anchor=north, font=\bfseries\scriptsize, inner sep=0pt]
        at ($(img.south west)!0.3125!(img.south east)+(0,-2pt)$) {DDS};

    \node[anchor=north, font=\bfseries\scriptsize, inner sep=0pt]
        at ($(img.south west)!0.4375!(img.south east)+(0,-2pt)$) {Prox-SNORE};

    \node[anchor=north, font=\bfseries\scriptsize, inner sep=0pt]
        at ($(img.south west)!0.5625!(img.south east)+(0,-2pt)$) {Ann-Prox-SNORE};

    \node[anchor=north, font=\bfseries\scriptsize, inner sep=0pt]
        at ($(img.south west)!0.6875!(img.south east)+(0,-2pt)$) {Prox ML-PnP};

    \node[anchor=north, font=\bfseries\scriptsize, inner sep=0pt]
        at ($(img.south west)!0.8125!(img.south east)+(0,-2pt)$) {ML-SPnP (ours)};

    \node[anchor=north, font=\bfseries\scriptsize, inner sep=0pt]
        at ($(img.south west)!0.9375!(img.south east)+(0,-2pt)$) {GT};
\end{tikzpicture}%
}

\newcommand{\boldc}{\bm{c}}
\newcommand{\boldd}{\bm{d}}
\newcommand{\bolde}{\bm{e}}

\newcommand{\boldn}{\bm{n}}

\newcommand{\boldx}{\bm{x}}
\newcommand{\boldy}{\bm{y}}
\newcommand{\boldz}{\bm{z}}

\newcommand{\boldzero}{\bm{0}}

\newcommand{\boldA}{\bm{A}}

\newcommand{\boldD}{\bm{D}}

\newcommand{\boldI}{\bm{I}}

\newcommand{\boldP}{\bm{P}}

\newcommand{\boldR}{\bm{R}}

\newcommand{\boldU}{\bm{U}}
\newcommand{\boldV}{\bm{V}}
\newcommand{\boldW}{\bm{W}}

\newcommand{\calN}{\mathcal{N}}

\newcommand{\calV}{\mathcal{V}}

\newcommand{\calX}{\mathcal{X}}
\newcommand{\calY}{\mathcal{Y}}

\newcommand{\rmd}{\mathrm{d}}

\DeclareMathOperator*{\argmin}{arg\,min}

\newcommand{\R}{\mathbb{R}}

\newcommand{\transp}{^\top}

\usepackage{comment}

\setabbreviationstyle{long-short}
\glsdisablehyper

\newabbreviation{2D}{2-D}{two-dimensional}
\newabbreviation{3D}{3-D}{three-dimensional}
\newabbreviation{dD}{$d$-D}{$d$-dimensional}
\newabbreviation{mD}{$m$-D}{$m$-dimensional}
\newabbreviation{AC}{AC}{attenuation correction}
\newabbreviation{AE}{AE}{absolute error}
\newabbreviation{ADMM}{ADMM}{alternating direction method of multipliers}
\newabbreviation{ALARA}{ALARA}{As Low As Reasonnably Achievable}
\newabbreviation{CAOL}{CAOL}{convolutional analysis operator learning}
\newabbreviation{CDL}{CDL}{convolutional DiL}
\newabbreviation{CNN}{CNN}{convolutional neural network}
\newabbreviation{CRC}{CRC}{contrast recovery coefficient}
\newabbreviation{CS}{CS}{compressive sensing}
\newabbreviation{CT}{CT}{computed tomography}
\newabbreviation{SVCT}{SVCT}{sparse-view computed tomography}
\newabbreviation{DDIM}{DDIM}{denoising diffusion implicit model}
\newabbreviation{DDPM}{DDPM}{denoising diffusion probabilistic model}
\newabbreviation{DECT}{DECT}{dual-energy computed tomography}
\newabbreviation{DiL}{DiL}{dictionary learning}
\newabbreviation{DL}{DL}{deep learning}
\newabbreviation{AI}{AI}{artificial intelligence}
\newabbreviation{DM}{DM}{diffusion model}
\newabbreviation{DPS}{DPS}{diffusion posterior sampling}
\newabbreviation{SD}{SD}{standard deviation}
\newabbreviation{SDE}{SDE}{stochastic differential equation}
\newabbreviation{DTV}{DTV}{directional total variation}
\newabbreviation{DWT}{DWT}{discrete wavelet transform}
\newabbreviation{EM}{EM}{expectation-maximization}
\newabbreviation{FBP}{FBP}{filtered backprojection}
\newabbreviation{FC}{FC}{full-count}
\newabbreviation{FWHM}{FWHM}{full width at half maximum}
\newabbreviation{GAN}{GAN}{generative adversarial network}
\newabbreviation{GT}{GT}{ground truth}
\newabbreviation{HC}{HC}{high-count}
\newabbreviation{HU}{HU}{Houndsfield Units}
\newabbreviation{FOV}{FOV}{field of view}
\newabbreviation{IDWT}{IDWT}{inverse discrete wavelet transform}
\newabbreviation{IFFT}{IFFT}{inverse fast Fourier transform}
\newabbreviation{JRAA}{JRAA}{joint reconstruction of activity and attenuation}
\newabbreviation{JTV}{JTV}{joint total variation}
\newabbreviation{KL}{KL}{Kullback-Leibler}
\newabbreviation{LC}{LC}{low-count}
\newabbreviation{LAC}{LAC}{linear attenuation coefficient}
\newabbreviation{LDM}{LDM}{latent diffusion model}
\newabbreviation{LBFGS}{L-BFGS}{limited-memory  Broyden-Fletcher-Goldfarb-Shanno}
\newabbreviation[
	longplural={lines of response},
	firstplural={lines of response (LORs)}
]{LOR}{LOR}{line of response}
\newabbreviation{MAP}{MAP}{maximum \emph{a~posteriori}}
\newabbreviation{MAPEM}{MAPEM}{maximum \textit{a posteriori} expectation maximization}
\newabbreviation{MBIR}{MBIR}{model-based iterative reconstruction}
\newabbreviation{MCDL}{CDL}{multichannel convolutional dictionary learning}
\newabbreviation{MCAOL}{MCAOL}{multichannel convolutional analysis operator learning}
\newabbreviation{MLAA}{MLAA}{maximum likelihood activity and attenuation}
\newabbreviation{ACF}{ACF}{attenuation correction factor}
\newabbreviation{MLACF}{MLACF}{maximum likelihood activity and attenuation correction factors}
\newabbreviation{MLEM}{MLEM}{maximum-likelihood expectation-maximization}
\newabbreviation{MNIST}{MNIST}{Modified National Institute of Standards and Technology}
\newabbreviation{MRI}{MRI}{magnetic resonance imaging}
\newabbreviation{MSE}{MSE}{mean squared error}
\newabbreviation{VOI}{VOI}{volume of interest}
\newabbreviation{MVAE}{MVAE}{multi-branch VAE}
\newabbreviation{NAC}{NAC}{non attenuation-corrected}
\newabbreviation{NN}{NN}{neural network}
\newabbreviation{NRMSE}{NRMSE}{normalized root mean squared error}
\newabbreviation{OOD}{OOD}{out-of-distribution}
\newabbreviation{OSEM}{OSEM}{ordered subsets expectation maximization}
\newabbreviation{PCA}{PCA}{principal component analysis}
\newabbreviation{PCCT}{PCCT}{photon-counting computed tomography}
\newabbreviation{PDF}{PDF}{probability distribution function}
\newabbreviation{PET}{PET}{positron emission tomography}
\newabbreviation{PLS}{PLS}{parallel level set}
\newabbreviation{PML}{PML}{penalized maximum likelihood}
\newabbreviation{PVE}{PVE}{partial volume effect}
\newabbreviation{PSNR}{PSNR}{peak signal-to-noise ratio}
\newabbreviation{PWLS}{PWLS}{penalized weighted least squares}
\newabbreviation{RC}{RC}{recovery coefficient}
\newabbreviation{ReLU}{RelU}{rectified linear unit}
\newabbreviation{ROI}{ROI}{region of interest}
\newabbreviation{SPECT}{SPECT}{single-photon emission CT}
\newabbreviation{SPS}{SPS}{separable paraboloidal surrogates}
\newabbreviation{SQS}{SQS}{separable quadratic surrogate}
\newabbreviation{SNR}{SNR}{signal-to-noise ratio}
\newabbreviation{SSIM}{SSIM}{structural similarity index measure}
\newabbreviation{LPIPS}{LPIPS}{learned perceptual image patch similarity}
\newabbreviation{SSS}{SSS}{single scatter simulation}
\newabbreviation{TNV}{TNV}{total nuclear variation}
\newabbreviation{TOF}{TOF}{time-of-flight}
\newabbreviation{TV}{TV}{total variation}
\newabbreviation{VAE}{VAE}{variational autoencoder}
\newabbreviation{beta-VAE}{$\beta$-VAE}{beta-variational autoencoder}
\newabbreviation{WDM}{WDM}{wavelet diffusion model}
\newabbreviation{WGAN}{W-GAN}{Wasserstein GAN}
\newabbreviation{WLS}{WLS}{weighted least squares}
\newabbreviation{STD}{STD}{standard deviation}
\newabbreviation{XCAT}{XCAT}{extended cardiac-torso}
\newabbreviation{FDG}{\textsuperscript{18}F-FDG}{[\textsuperscript{18}F]Fluorodeoxyglucose}
\newabbreviation{PSMA}{[\textsuperscript{18}F]-PSMA}{[\textsuperscript{18}F]-\emph{prostate-specific membrane antigen-11}}
\newabbreviation{PnP}{PnP}{Plug-and-Play}
\newabbreviation{RED}{RED}{regularization by denoising}
\newabbreviation{MMSE}{MMSE}{minimum mean-square error}
\newabbreviation{SNORE}{SNORE}{Stochastic deNOising REgularization}

\newabbreviation{MRA}{MRA}{multiresolution analysis}
\newabbreviation{DPIR}{DPIR}{Deep Plug-and-Play for Image
Restoration}
\newabbreviation{DDS}{DDS}{Decomposed Diffusion Sampler}
\newabbreviation{DiffPIR}{DiffPIR}{Diffusion Plug-and-Play Image Restoration}
\newabbreviation{ML}{ML}{multilevel}
\newabbreviation{MR}{MR}{multiresolution}
\newabbreviation{NFE}{NFE}{number of function evaluation} 

\begin{document}
	
	\title{Multilevel Stochastic Plug-and-Play for Sparse-View CT Reconstruction}
	
	\author{
		Antoine De Paepe, Alexandre Bousse, Dimitris Visvikis
		\thanks{
			This work was supported by CPER 2021--2027 IMAGIIS (INNOV-XS) and by the Région Bretagne through the ARED doctoral research grant program.
			}
		\thanks{
			All authors  are affiliated with the LaTIM, Inserm, UMR 1101, \emph{Universit\'e de Bretagne Occidentale}, Brest, France.
			}
		\thanks{
			Corresponding author: A. Bousse; email: \href{mailto:bousse@univ-brest.fr}{\texttt{bousse@univ-brest.fr}}.
		      }	
		}
	
	\maketitle
	
	

	
	
	
    \begin{abstract}
	\Gls{SVCT} reduces radiation exposure and acquisition time, but the limited number of projection views makes the reconstruction problem severely ill-posed and leads to streak artifacts when analytical methods are used. \Gls{PnP} methods combine data fidelity with learned image priors, while stochastic \gls{PnP} methods use re-noising to reduce the mismatch between denoiser training inputs and reconstruction iterates. However, these methods often require many iterations, which limits their practical efficiency. In this work, we propose \gls{ML} Stochastic Plug-and-Play (ML-SPnP), a \gls{ML} stochastic \gls{PnP} method for \gls{SVCT} reconstruction formulated on wavelet \gls{MRA} approximation spaces. The method combines coarse-to-fine initialization with a finite sequence of \gls{ML} corrections using resolution-specific denoisers. Under suitable assumptions, the fine-level iterations following the finite \gls{ML} phase inherit the convergence guarantees of the underlying stochastic \gls{PnP} algorithm. Experiments on two CT datasets with 20--50 projection views show that ML-SPnP achieves reconstruction quality comparable to the evaluated stochastic \gls{PnP} methods, while generally requiring less reconstruction time. These results support \gls{ML} acceleration as an effective strategy for reducing the computational cost of stochastic \gls{PnP} reconstruction. The code will be made publicly available at \url{https://github.com/antoinedepaepe/ml-spnp}.
\end{abstract}

	\section{Introduction}

\glsunset{CT} 

\IEEEPARstart{C}{omputed} tomography (CT) is a widely used imaging modality in medical diagnosis, allowing the recovery of internal anatomical structures from X-ray projection measurements. While conventional \gls{CT} relies on many projection views to obtain high-quality reconstructions, this may increase radiation exposure and acquisition time \cite{brenner2007computed}. To reduce this burden, several strategies have been proposed, including lowering the X-ray dose or reducing the number of projection views through sparse-view acquisition \cite{mccollough2009strategies}.

In \gls{SVCT}, the reduced number of projection angles makes the image reconstruction problem severely ill-posed. In that context, traditional analytical reconstruction methods such as \gls{FBP} \cite{natterer2001mathematics} often suffer from severe streak artifacts, which can obscure anatomical structures and complicate clinical interpretation. To overcome these limitations, \gls{MBIR} and \gls{CS} methods \cite{erdogan2002monotonic,sidky2008image,chambolle2011first,niu2014sparse} incorporate prior information into the reconstruction process, most commonly through handcrafted regularizers, such as \gls{TV} regularization. Although \gls{TV}-based approaches can reduce streak artifacts, they often oversmooth images and may remove fine structural details.

These limitations have motivated the development of deep learning-based methods for \gls{SVCT}. In image-domain post-processing, an initial reconstruction, typically obtained with \gls{FBP}, is refined by a neural network to suppress streak artifacts and restore anatomical details \cite{chen2017low,jin2017deep}. Unrolling methods embed the acquisition model and raw projection measurements into the network architecture by alternating learned reconstruction updates with explicit data-consistency steps \cite{adler2018learned,rudzusika20243d,vo2026efficient}. 
While these methods demonstrate strong reconstruction performance, models trained for a particular acquisition setting may lose performance when that setting changes. Recent work has improved adaptability by encoding projection-view distributions \cite{ma2026universal}, representing acquisition geometry explicitly \cite{valat2025improving}, adapting pretrained reconstruction models at test time \cite{song2023piner}, and conditioning reconstruction on the acquisition physics with optional self-supervised fine-tuning \cite{terris2025reconstruct}.

To address this lack of flexibility, recent works have explored generative priors that can be combined with different degradation or acquisition models. In particular, diffusion models \cite{ho2020denoising} have emerged as powerful generative priors for inverse problem solving \cite{chung2022diffusion,zhu2023denoising}, with several extensions for \gls{CT} reconstruction \cite{liu2023dolce,chung2023solving,chung2024decomposed,li2024ct}. Similarly, flow-matching models \cite{lipman2022flow} offer another approach to learning generative priors for image reconstruction \cite{hashimoto2026flow}. However, despite their strong generative capabilities, diffusion- and flow-matching-based methods may hallucinate realistic-looking structures in severely ill-posed settings \cite{tivnan2024hallucination}, which is particularly concerning in medical imaging. 

\Gls{PnP} methods provide a natural framework for this purpose by embedding a learned denoiser into an iterative reconstruction algorithm, where the data-fidelity step enforces consistency with the measurements and the denoiser acts as an image prior \cite{venkatakrishnan2013plug,chan2016plug,romano2017little}. Different \gls{PnP} variants mainly differ in how this denoising step is interpreted. In \gls{ADMM}-based \gls{PnP}, the proximal operator of a regularizer is replaced by a denoiser, yielding an implicit prior model \cite{zhang2021plug}. By contrast, \gls{RED} defines an explicit regularization term directly from the denoiser \cite{romano2017little}. Other approaches further constrain the denoiser to correspond to a gradient step or a proximal map of an explicit regularizer, leading to convergent \gls{PnP} algorithms with a clear variational interpretation \cite{hurault2021gradient,hurault2022proximal}. More recently, stochastic denoising formulations have been introduced to reduce the mismatch between the denoiser training distribution and the inputs encountered during reconstruction \cite{renaud2024plug,park2026stochastic,martin2025pnp}. Although most of these methods were originally developed for image restoration, they can be naturally extended to \gls{CT} reconstruction \cite{denker2026stability,tatachak2026gradient}.

Despite their flexibility, \gls{PnP} methods can be computationally demanding, particularly in \gls{CT}, where convergence often requires many costly reconstruction iterations. Several strategies have therefore been proposed to accelerate their convergence. For instance, methods such as \gls{DPIR} \cite{zhang2021plug} and annealed \gls{SNORE} \cite{renaud2024plug} vary the denoiser noise level along the iterations, acting as a coarse-to-fine regularization schedule. In parallel, \gls{ML} optimization methods have recently been proposed to accelerate  large-scale inverse problem algorithms by transferring information across a hierarchy of resolutions \cite{nash2000multigrid,lauga2023multilevel,lauga2024iml}. Building on this idea, \gls{ML}-\gls{PnP} extends \gls{ML} acceleration to \gls{PnP} image restoration \cite{laurent2025multilevel}. This raises the question of whether \gls{ML} acceleration can also reduce the computational cost of stochastic \gls{PnP} methods such as Prox-\gls{SNORE} \cite{renaud2025convergence}.

In this work, we propose to extend the \gls{ML}-\gls{PnP} methods to the stochastic \gls{PnP} regularizer for \gls{SVCT} reconstruction. To this end, we introduce \emph{Multilevel Stochastic Plug-and-Play} (ML-SPnP), an \gls{ML} stochastic \gls{PnP} algorithm formulated on wavelet \gls{MRA} spaces, inspired by our previous work \cite{DePaepe2025a, PhungNgoc2025}. The main contributions of this work are as follows:
\begin{itemize}
	\item \textbf{Stochastic \Gls{ML} \gls{PnP}:} We introduce, to the best of our knowledge, the first \gls{ML} reconstruction algorithm based on renoising/denoising-based stochastic \gls{PnP} regularization. 
	\item \textbf{Efficient \gls{SVCT} Reconstruction:} Experiments on two medical \gls{CT} datasets demonstrate that
    ML-SPnP outperforms the evaluated diffusion-based methods
    on standard image-quality metrics and matches the reconstruction
    quality of the compared renoising/denoising-based stochastic
    \gls{PnP} methods with shorter runtimes. It also improves perceptual similarity over the evaluated
    \gls{ML}-\gls{PnP} method, at the cost of longer runtimes.
\end{itemize}

The remainder of this paper is organized as follows: Section~\ref{sec:background} presents the stochastic \gls{PnP} and \gls{ML} optimization frameworks; Section~\ref{sec:methods} introduces the proposed ML-SPnP method; Section~\ref{sec:experiments} describes the experimental setup and presents the results; Section~\ref{sec:discussion} discusses the implications and limitations of the approach; finally, Section~\ref{sec:conclusion} concludes this work.

	\section{Background}\label{sec:background}

\subsection{Inverse Problem}

In \gls{SVCT}, the objective is to reconstruct an attenuation image $\boldx \in \calX \coloneq \R^n$ from projection measurements $\boldy \in \calY \coloneq \R^m$. The image contains $n = 2^{dN}$ pixels or voxels ($d=2$ for \gls{2D} and $d=3$ for \gls{3D} images), with $2^N$ pixels or voxels along each axis (e.g., $N=9$ gives 512 along each axis). The measurements comprise $m = n_\theta \cdot n_\rmd$ detector-bin values, where $n_\theta$ is the number of projection angles and $n_\rmd$ is the number of detector bins per angle. In a simplified monochromatic (post-log) setting, the measurement $\boldy$ relates to the image $\boldx$ through
\begin{equation}\label{eq:forward_pb}
	\boldy = \boldA \boldx + \bolde,
\end{equation} 
where $\boldA$ is the X-ray projection operator and 
$\bolde \sim \calN(\boldzero_{\calY}, \bold{\Sigma})$ is the measurement noise, with 
$\bold{\Sigma}$ denoting the diagonal covariance of the measurement-wise variances. The image $\boldx$ is assumed to be a random vector with prior \gls{PDF} $p(\boldx)$, and $p(\boldy | \boldx)$ is determined by \eqref{eq:forward_pb}.

In principle, the image $\boldx$ can be reconstructed from $\boldy$ by applying an approximate inverse of the projection operator, such as the pseudo-inverse $\boldA^\dag \boldy$, which is commonly implemented in \gls{CT} through analytical methods such as \gls{FBP} \cite{natterer2001mathematics}, or by solving an optimization problem of the form
\begin{equation}\label{eq:inverse_problem}
	\min_{\boldx \in \calX} \, h(\boldx) = f(\boldx) +  \lambda g(\boldx)
\end{equation}
where $f(\boldx) = d(\boldA\boldx,\boldy)$ is the data fidelity term, $d$ measures the discrepancy between $\boldA\boldx$ and $\boldy$, $g$ is a regularizer and $\lambda>0$ is a parameter that controls its strength. 
For $\lambda=1$, this formulation corresponds to \gls{MAP} estimation under the stated likelihood and prior 
\begin{equation}
	f(\boldx)=-\log p(\boldy|\boldx), \qquad g(\boldx)=-\log p(\boldx) \, .
\end{equation}
However, in practice, the prior $p(\boldx)$ is intractable and $g(\boldx)$ is often replaced by a handcrafted regularizer such as \gls{TV} \cite{chambolle2011first}.  

Problems of the form \eqref{eq:inverse_problem} are commonly solved using proximal splitting algorithms \cite{parikh2014proximal}, which decouple the treatment of the data-fidelity and regularization terms. When $f$ is differentiable, the standard proximal-gradient update performs a gradient step on the data-fidelity term from the current estimate $\boldx^{(k)}$ at iteration $k$, followed by a proximal step on the regularizer:
\begin{align} 
	\boldx^{(k+1)} & {} = \operatorname{prox}_{\delta_k \lambda g}\left(\boldx^{(k)} - \delta_k \nabla f\left(\boldx^{(k)}\right)\right)  \label{eq:prox} 
\end{align}
where $\operatorname{prox}_{h}\colon \boldx \mapsto \argmin_{\boldz} \frac{1}{2} \|\boldz-\boldx\|_2^2 + h(\boldz)$ is the proximal operator associated to the convex proper function $h$ and $\delta_k>0$ is the iteration-dependent step size.

\subsection{PnP priors}

\Gls{PnP} methods \cite{venkatakrishnan2013plug,chan2016plug,zhang2021plug} avoid specifying the prior $g$ explicitly by replacing the proximal operator of $g$ in \eqref{eq:prox} with an image denoiser $\boldD_{\sigma}$, leading to updates of the form 
\begin{equation}
    \boldx^{(k+1)} = \boldD_{\sigma} \left( \boldx^{(k)} - \delta_k \nabla f\left(\boldx^{(k)} \right) \right) \, .
\end{equation}
A related interpretation is provided by \gls{RED} \cite{romano2017little} for Gaussian noise. 
Given a noisy version of $\boldx$, $\boldz=\boldx+\sigma\boldn$  with $\boldn\sim\mathcal{N}(\boldzero_{\calX},\boldI_{\calX})$ and $\sigma>0$, the \gls{MMSE} denoiser $\boldD_\sigma(\boldz)=\mathbb{E}[\boldx|\boldz]$ satisfies Tweedie's formula $\boldD_\sigma(\boldz)=\boldz+\sigma^2\nabla \log p_\sigma(\boldz)$, where $p_\sigma = p \ast \calN(\boldzero_\calX , \sigma^2 \boldI_\calX)$ is the \gls{PDF} of $\boldz$. Rearranging gives
\begin{equation}\label{eq:gradient}
	\nabla  g_\sigma(\boldz) =\frac{1}{\sigma^2} \left( \boldz-\boldD_\sigma(\boldz) \right),
\end{equation}
with $g_\sigma :=  -\log p_\sigma$. Thus, the \gls{RED} denoising residual can be interpreted as the gradient
of a smoothed negative log-prior. The resulting reverse proximal
splitting update is given by
\begin{equation}
    \boldx^{(k+1)} = \operatorname{prox}_{\delta_k f} \left( \boldx^{(k)} - \frac{\delta_k \lambda}{\sigma^2} \left( \boldx^{(k)} - \boldD_{\sigma}\left(\boldx^{(k)}\right)\right)\right),
\end{equation}
where a gradient step on the prior is followed by a proximal step
on the data-fidelity term.

\subsection{Stochastic PnP priors}

Stochastic \gls{PnP} methods \cite{renaud2024plug, martin2025pnp, park2026stochastic} extend the denoiser-prior connection by explicitly injecting randomness into the regularization step. This is motivated by the fact that most denoisers are trained for Gaussian denoising, while classical \gls{PnP} and \gls{RED} often apply them to iterates whose effective noise level is not controlled. To reduce this mismatch, one can re-noise the current iterate before applying the denoiser. This corresponds to redefining the negative log-prior as 
\begin{equation}\label{eq:snore_prior}
	g_\sigma (\boldx):= - \mathbb{E}_{\boldn \sim \calN(\boldzero_\calX , \boldI_\calX)}\left[ \log \, p_\sigma (\boldx + \sigma \boldn)   \right]\,.
\end{equation}
This idea is used in \gls{SNORE} \cite{renaud2024plug,renaud2025convergence}, which defines stochastic iterates as
\begin{equation}\label{eq:proxsnore}
    \begin{aligned}
        \boldz^{(k+1)} &= \boldx^{(k)}
        - \frac{\delta \lambda}{\sigma^2}
        \left( \boldx^{(k)} - \boldD_{\sigma}
        \left(\boldx^{(k)} + \sigma \boldn^{(k)}\right) \right), \\
        \boldx^{(k+1)} &= \operatorname{prox}_{\delta f}
        \left(\boldz^{(k+1)}\right),
    \end{aligned}
\end{equation}
where $\boldn^{(k)} \sim \mathcal{N}(\boldzero_{\calX},\boldI_{\calX})$ and $\delta>0$.
This preserves the data-fidelity proximal step while replacing the deterministic prior-gradient step with a stochastic estimator of the gradient of a noise-averaged regularizer, defined in Section~\ref{sec:mlpnp}.

\subsection{Multilevel optimization}

\Gls{ML} optimization \cite{hackbusch2013multi, parpas2017multilevel, ho2022newton} is a class of computational methods designed to accelerate the solving of optimization problems by exploiting a hierarchy of related objective functions at different resolutions. Rather than optimizing an objective function only on the finest, highest-dimensional space, these methods use cheaper coarse-level problems to compute correction steps that are then transferred back to the fine level. Since coarse problems involve fewer unknowns, they are less expensive to solve, often leading to faster convergence.

In the basic two-level setting, let $h_0(\boldx_0)$ be the fine objective function defined on $\calX_0 \coloneq \calX = \R^n$, i.e., $h_0 = h$ as defined in  \eqref{eq:inverse_problem}, and $h_1(\boldx_1)$ be a coarse objective function approximating $h_0(\boldx_0)$ on a coarse image space $\calX_1 \coloneq \R^{n/2^d}$, usually defined as $h_1(\boldx_1) = f(\boldP_1 \boldx_1) + \lambda g(\boldP_1 \boldx_1)$, where $\boldP_1\colon \calX_1\to \calX_0$ is the linear prolongation operator with associated reduction operator $\boldR_1 \propto \boldP_1\transp\colon \calX_0\to \calX_1$. Given a fine estimate $\boldx_0^{(k)} = \boldx^{(k)}\in\calX_0$ at iteration $k$, the coarse model is often corrected using a surrogate function $\phi_1^{(k)}$ defined as
\begin{equation}
	\phi_1^{(k)}(\boldx_1) = h_1(\boldx_1) + \left\langle\boldc_1^{(k)},\,\boldx_1\right\rangle ,
\end{equation}
where $\boldc_1^{(k)}$ is chosen 
so that the corrected coarse model is first-order coherent \cite{wen2010line} with the fine objective at the restricted iterate:
\begin{equation}
	\nabla \phi_1^{(k)}\left(\boldR_1\boldx^{(k)}\right) = {}  \boldR_1 \nabla h_0\left(\boldx^{(k)}\right) \,.
\end{equation} 
The surrogate function is approximately minimized on the coarse space $\calX_1$ by performing $K_1$ iterations of an iterative algorithm initialized at $\boldx_1^{(0,k)} = \boldR_1 \boldx_0^{(k)}$. The resulting coarse correction is then transferred to the fine space to update the current estimate:
\begin{equation}
    \boldx_0^{(k+1)}
    = \boldx_0^{(k)}
    + \tau_k \boldP_1
    \left(\boldx_1^{(K_1,k)} - \boldx_1^{(0,k)}\right),
\end{equation}
where $\tau_k>0$ is a step size chosen to ensure sufficient decrease of the fine-level objective.
This mechanism ensures that coarse-level models are locally consistent with the fine-level optimization problem, so that coarse solves provide meaningful corrections rather than independent approximations. Applying this correction principle across a hierarchy of resolutions yields a \gls{ML} optimization scheme, for which a V-cycle is a common choice of level traversal \cite{briggs2000multigrid}. These methods have also been adapted to imaging inverse problems, including computed tomography \cite{bouman1992nonlinear,plier2021first,elshiaty2026multilevel} and image restoration \cite{lauga2024iml,laurent2025multilevel}.
	\section{Methods}\label{sec:methods}

In the following, we propose the ML-SPnP method for accelerating the iterative solution of \eqref{eq:inverse_problem}, with $g=g_\sigma$ defined in \eqref{eq:snore_prior}.

\subsection{Multilevel optimization with a stochastic PnP prior}\label{sec:mlpnp}

To formulate the \gls{ML} optimization problem, we first specify the sequence of optimization spaces on which the different levels are defined. We consider a hierarchy of $L$ levels, $\{\calX_\ell\}_{\ell=0}^{L-1}$, with $\calX_\ell = \R^{n/2^{d\ell}}$, where $\calX_0=\calX$ corresponds to the finest image representation, whereas larger values of $\ell$ correspond to progressively coarser representations. For each $\ell=1,\dots,L-1$, we introduce a linear restriction operator $\boldR_\ell\colon\calX_{\ell-1}\to\calX_\ell$ and the associated prolongation operator $\boldP_\ell\propto\boldR_\ell\transp\colon\calX_\ell\to\calX_{\ell-1}$. The cumulative restriction and prolongation operators are
\begin{equation}
    \bar{\boldR}_\ell \coloneqq \boldR_\ell \cdots \boldR_1,
    \qquad
    \bar{\boldP}_\ell \coloneqq \boldP_1 \cdots \boldP_\ell
    = \bar{\boldR}_\ell\transp,
\end{equation}
with $\bar{\boldR}_0=\bar{\boldP}_0=\boldI_{\calX_0}$. In our implementation, $\boldR_\ell$ extracts the approximation coefficients of an orthonormal \gls{2D} Haar wavelet transform \cite{mallat1999wavelet}, and $\boldP_\ell$ is the corresponding synthesis operator with zero detail coefficients, so that $\boldP_\ell=\boldR_\ell\transp$. Given a fine-level estimate $\boldx_0^{(k)}$ at iteration $k$, we denote by $\boldx_\ell^{(k)}\coloneqq\bar{\boldR}_\ell\boldx_0^{(k)}$ its restriction to level $\ell$. At each level $\ell=0,\dots,L-1$, we define the objective
\begin{equation}\label{eq:obj_l}
    h_{\sigma,\ell}(\boldx_\ell)
    \coloneqq f_\ell(\boldx_\ell)
    + \lambda g_{\sigma,\ell}(\boldx_\ell),
\end{equation}
where $f_\ell(\boldx_\ell)\coloneqq f(\bar{\boldP}_{\ell}\boldx_\ell) = d(\boldA\bar{\boldP}_{\ell}\boldx_\ell,\boldy)$ and $g_{\sigma,\ell}$ is a regularizer at level $\ell$ that depends on a parameter $\sigma>0$. We note $h_{\sigma, 0} = h_\sigma$, $f_{0} = f$ and $g_{\sigma, 0} = g_\sigma$.

Following \gls{SNORE} \cite{renaud2024plug}, we define this regularizer through a stochastic \gls{PnP} prior evaluated at Gaussian perturbations of the optimization variable. In the proposed \gls{ML} setting, this construction is applied separately at each wavelet resolution. Let $p_\ell\coloneqq(\bar{\boldR}_\ell)_{\#}p$ be the pushforward of the finest image prior $p$ to level $\ell=0,\dots,L-1$, so that $p_0=p$, and let
\begin{equation}
    p_{\ell,\sigma}
    \coloneqq p_\ell *
    \mathcal N(\boldzero_{\calX_\ell},\sigma^2\boldI_{\calX_\ell})
\end{equation}
be its Gaussian-smoothed version. For $\boldn_\ell\sim\mathcal N(\boldzero_{\calX_\ell},\boldI_{\calX_\ell})$, the stochastic prior at level $\ell$ is defined by
\begin{equation}\label{eq:gsigma_exp}
    g_{\sigma,\ell}(\boldx_\ell)
    \coloneqq
    -\mathbb E_{\boldn_\ell}
    \left[
        \log p_{\ell,\sigma}(\boldx_\ell+\sigma\boldn_\ell)
    \right].
\end{equation}
The corresponding stochastic vector field is
\begin{equation}\label{eq:stoch_field}
    \boldU_{\sigma,\ell}(\boldx_\ell,\boldn_\ell)
    \coloneqq
    \frac{1}{\sigma^{2}}
    \left(
        \boldx_\ell
        - \boldD_{\sigma,\ell}(\boldx_\ell+\sigma\boldn_\ell)
    \right),
\end{equation}
where $\boldD_{\sigma,\ell}$ is the \gls{MMSE} denoiser at level $\ell$, with $\boldU_{\sigma,0}=\boldU_\sigma$ and $\boldD_{\sigma,0}=\boldD_\sigma$. By Tweedie's formula,
\begin{equation}\label{eq:gsigma_exp_grad}
    \nabla g_{\sigma,\ell}(\boldx_\ell)
    =
    \mathbb E_{\boldn_\ell}
    \left[
        \boldU_{\sigma,\ell}(\boldx_\ell,\boldn_\ell)
    \right].
\end{equation}
Consequently,
$\nabla f_\ell(\boldx_\ell)
+\lambda\boldU_{\sigma,\ell}(\boldx_\ell,\boldn_\ell)$
is an unbiased stochastic gradient estimator of the objective in \eqref{eq:obj_l}.

Having specified the level objectives, we introduce a linear correction to ensure first-order coherence across the hierarchy. Given a fine-level estimate $\boldx_0^{(k)}=\boldx^{(k)}\in\calX_0$ at iteration $k$, the corrected surrogate at level $\ell$ is
\begin{equation}\label{eq:surrogate}
    \phi_{\sigma,\ell}^{(k)}(\boldx_\ell)
    \coloneqq
    h_{\sigma,\ell}(\boldx_\ell)
    + \left\langle
        \boldc_\ell^{(k)},\boldx_\ell
    \right\rangle,
\end{equation}
where $\boldc_0^{(k)}=\boldzero_{\calX_0}$ and, for $\ell=1,\dots,L-1$,
\begin{equation}\label{eq:cl}
    \boldc_\ell^{(k)}
    \coloneqq
    \boldR_\ell
    \nabla\phi_{\sigma,\ell-1}^{(k)}
    \left(\boldx_{\ell-1}^{(k)}\right)
    -
    \nabla h_{\sigma,\ell}
    \left(\boldR_\ell\boldx_{\ell-1}^{(k)}\right).
\end{equation}
This construction ensures that the corrected gradients satisfy the first-order coherence condition
\begin{equation}\label{eq:coherence_l}
    \nabla\phi_{\sigma,\ell+1}^{(k)}
    \left(\boldR_{\ell+1}\boldx_\ell^{(k)}\right)
    =
    \boldR_{\ell+1}
    \nabla\phi_{\sigma,\ell}^{(k)}
    \left(\boldx_\ell^{(k)}\right),
\end{equation}
for $\ell=0,\dots,L-2$. In other words, at the restricted iterate, the gradient of the corrected coarse objective agrees with the restriction of the corrected gradient from the preceding level. This is the mechanism by which descent information is transferred coherently across the \gls{ML} hierarchy. Applying \eqref{eq:cl} recursively expresses the correction vector explicitly as a function of the fine-level iterate $\boldx_0^{(k)}$
\begin{equation}
    \boldc_\ell^{(k)}
    \coloneqq
    \boldc_\ell\left(\boldx_0^{(k)}\right)
    =
    \bar{\boldR}_\ell
    \nabla h_{\sigma,0}
    \left(\boldx_0^{(k)}\right)
    -
    \nabla h_{\sigma,\ell}
    \left(\bar{\boldR}_\ell\boldx_0^{(k)}\right).
\end{equation}
Thus, the correction vector decomposes as
$\boldc_\ell^{(k)} =\boldc_{f,\ell}^{(k)} + \lambda \boldc_{g,\ell}^{(k)}$,
where
\begin{equation}\label{eq:cf}
    \boldc_{f,\ell}^{(k)}
    \coloneqq
    \boldc_{f,\ell}\left(\boldx_0^{(k)}\right)
    =
    \bar{\boldR}_\ell
    \nabla f_0\left(\boldx_0^{(k)}\right)
    -
    \nabla f_\ell
    \left(\bar{\boldR}_\ell\boldx_0^{(k)}\right)
\end{equation}
enforces first-order coherence of the data-fidelity term, while
\begin{equation}\label{eq:cg}
    \boldc_{g,\ell}^{(k)}
    \coloneqq
    \boldc_{g,\ell}\left(\boldx_0^{(k)}\right)
    =
    \bar{\boldR}_\ell
    \nabla g_{\sigma,0}
    \left(\boldx_0^{(k)}\right)
    -
    \nabla g_{\sigma,\ell}
    \left(\bar{\boldR}_\ell\boldx_0^{(k)}\right)
\end{equation}
enforces first-order coherence of the prior term.

In practice, the prior correction in \eqref{eq:cg} can be estimated using the stochastic vector fields in \eqref{eq:stoch_field}, introducing sampling variability and additional denoiser evaluations. Our ablation study in Section~\ref{sec:ablation} shows similar reconstruction quality with and without this correction in the evaluated configuration. We therefore omit the prior correction in our experiments with ML-SPnP and retain only the data-fidelity coherence term.

\subsection{Stochastic proximal-gradient updates}

At each resolution level, we use a stochastic proximal-gradient update, treating the stochastic \gls{PnP} vector field explicitly and the corrected data-fidelity term through a proximal map.
At each level $\ell$ and outer iteration $k$, we approximately minimize the surrogate $\phi_{\sigma,\ell}^{(k)}$ defined in~\eqref{eq:surrogate} using $K_\ell$ stochastic proximal-gradient updates. Starting from $\boldx_\ell^{(0,k)}$, these updates read
\begin{align}
    \boldz_\ell^{(j+1,k)}
    &= \boldx_\ell^{(j,k)}
    - \lambda\delta_\ell
    \boldU_{\sigma,\ell}
    \left(\boldx_\ell^{(j,k)},\boldn_\ell^{(j,k)}\right), \\
    \boldx_\ell^{(j+1,k)}
    &= \operatorname{prox}_{\delta_\ell
        \left(f_\ell+\langle\boldc_\ell^{(k)},\cdot\rangle\right)}
    \left(\boldz_\ell^{(j+1,k)}\right),
\end{align}
for $j=0,\ldots,K_\ell-1$, where
$\boldn_\ell^{(j,k)}
\sim\calN(\boldzero_{\calX_\ell},\boldI_{\calX_\ell})$
and $\delta_\ell$ is a level-dependent step size.
Thus, each coarse problem is solved only approximately, using a small number of stochastic proximal-gradient steps.

\begin{algorithm}[H]
    \small
    \caption{ML-SPnP (main algorithm)}
    \label{algo:wml_spnp}
    \begin{algorithmic}[1]
        \Require Initial estimate $\boldA^\dag\boldy$,
        number of levels $L$, total iterations $K$,
        number of \gls{ML} iterations $K_{\rm ML}$,
        prior strength $\lambda$, noise level $\sigma$,
        step size $\delta>0$
        \State $\boldx^{(0)}
            \gets \operatorname{MR}(\boldA^\dag\boldy)$
        \For{$k=0,\dots,K-1$}
            \State $\widehat{\boldx}^{(k)}
                \gets
                \begin{cases}
                    \operatorname{ML}\left(\boldx^{(k)}\right),
                    & k<K_{\rm ML},\\
                    \boldx^{(k)},
                    & k\ge K_{\rm ML}.
                \end{cases}$
            \State $\boldn^{(k)}
                \sim\calN(\boldzero_{\calX},\boldI_{\calX})$
            \State $\boldz^{(k+1)}
                \gets \widehat{\boldx}^{(k)}
                - \lambda\delta
                \boldU_\sigma
                \left(\widehat{\boldx}^{(k)},\boldn^{(k)}\right)$
            \State $\boldx^{(k+1)}
                \gets \operatorname{prox}_{\delta f}
                \left(\boldz^{(k+1)}\right)$
                \label{step:prox_fine_level}
        \EndFor
        \State \Return $\boldx^{(K)}$
    \end{algorithmic}
\end{algorithm}

\begin{algorithm}[H]
    \small
    \caption{ML correction (subroutine)}
    \label{algo:ml}
    \begin{algorithmic}[1]
        \Require Fine-level iterate
        $\boldx^{(k)}=\boldx_0^{(k)}$ at iteration $k$,
        number of levels $L$,
        iteration numbers per level $K_\ell$,
        prior strength $\lambda$, noise level $\sigma$,
        coarse-level step sizes $\delta_\ell>0$,
        correction step sizes $\tau_{k,\ell}>0$

        \For{$\ell=1,\ldots,L-1$}
            \State $\boldx_\ell^{(k)}
                \gets \boldR_\ell\boldx_{\ell-1}^{(k)}$
            \State $\boldc_\ell^{(k)}
                \gets \boldc_{f,\ell}(\boldx_0^{(k)})
                + \lambda\boldc_{g,\ell}(\boldx_0^{(k)})$
        \EndFor

        \State $\widetilde{\boldx}_{L-1}
            \gets \boldx_{L-1}^{(k)}$

        \For{$\ell=L-1,\ldots,1$}
            \State $\boldx_\ell^{(0,k)}
                \gets \widetilde{\boldx}_\ell$

            \For{$j=0,\ldots,K_\ell-1$}
                \State $\boldn_\ell^{(j,k)}
                    \sim \calN
                    (\boldzero_{\calX_\ell},\boldI_{\calX_\ell})$
                \State $\boldz_\ell^{(j+1,k)}
                    \gets \boldx_\ell^{(j,k)}
                    - \lambda\delta_\ell
                    \boldU_{\sigma,\ell}
                    \left(
                        \boldx_\ell^{(j,k)},
                        \boldn_\ell^{(j,k)}
                    \right)$
                \State $\boldx_\ell^{(j+1,k)}
                    \gets \operatorname{prox}_{\delta_\ell
                        \left(
                            f_\ell
                            + \langle\boldc_\ell^{(k)},\cdot\rangle
                        \right)}
                    \left(\boldz_\ell^{(j+1,k)}\right)$
            \EndFor

            \State $\widetilde{\boldx}_{\ell-1}
                \gets \boldx_{\ell-1}^{(k)}
                + \tau_{k,\ell}\boldP_\ell
                \left(
                    \boldx_\ell^{(K_\ell,k)}
                    - \boldx_\ell^{(k)}
                \right)$
        \EndFor

        \State \Return $\widetilde{\boldx}_0$
    \end{algorithmic}
\end{algorithm}
\subsection{Algorithm Summary}

ML-SPnP is summarized in Algorithm~\ref{algo:wml_spnp}.
During the first $K_{\mathrm{ML}}$ iterations, the algorithm alternates between a \gls{ML} subroutine, described in Algorithm~\ref{algo:ml}, and a prox-\gls{SNORE} update~\eqref{eq:proxsnore} at the fine level $\ell=0$.
The proximal operations can be implemented with iterative solvers initialized from the previous estimate; we used a standard conjugate gradient method.



Inspired by the \gls{MR} initialization of \gls{ML}-\gls{PnP}~\cite{laurent2025multilevel}, we first refine the initial reconstruction from the coarsest to the finest resolution, using  stochastic proximal updates and prolongating the estimate between successive levels, illustrated in Algorithm~\ref{algo:mr}. This initialization aims to stabilize the early reconstruction steps and provides the starting point for the subsequent \gls{ML} updates. 

\begin{algorithm}[H]
	\small
	\caption{MR  (initialization)}\label{algo:mr}
	\begin{algorithmic}[1]
		\Require Initial fine-level estimate $\boldx_0 = \boldx$, number of levels $L$,
		number of iterations per level $K_\ell$, prior strength $\lambda$, noise level $\sigma$, coarse-level step sizes $\delta_{\ell} > 0$
		\State $\widetilde{\boldx}_{L-1} \gets \bar{\boldR}_{L-1}\boldx_0$
		\For{$\ell=L-1,\dots,0$}
			\State $\boldx_\ell^{(0)} \gets \widetilde{\boldx}_\ell$   
			\For{$j=0,\dots,K_{\ell}-1$}
				\State $\boldn_\ell^{(j)} \sim \calN( \boldzero_{\calX_\ell} , \boldI_{\calX_\ell} )  $
				\State $\boldz_{\ell}^{(j+1)} \gets \boldx_\ell^{(j)} - \lambda\delta_{\ell} \boldU_{\sigma,\ell}\left(\boldx_\ell^{(j)},\boldn_\ell^{(j)}\right)$
				\State $\boldx_\ell^{(j+1)} 	\gets \operatorname{prox}_{\delta_{\ell} f_\ell} \left(\boldz_{\ell}^{(j+1)}\right)$
			\EndFor
			\If{$\ell>0$}
				\State $\widetilde{\boldx}_{\ell-1} \gets \boldP_{\ell} \boldx_\ell^{(K_{\ell})}$ 
			\EndIf
		\EndFor
		\State \Return $\boldx_0^{(K_0)}$
	\end{algorithmic}
\end{algorithm}

\subsection{Convergence guarantees}\label{sec:convergence}

We provide a convergence justification of Algorithm~\ref{algo:wml_spnp} (with exact proximal evaluations) by noting that the \gls{ML} block in Algorithm~\ref{algo:wml_spnp} is activated only for a finite number of iterations. After this warm-up phase, the fine level algorithm coincides with the Prox-\gls{SNORE} one. Hence, under the assumptions of Prox-\gls{SNORE}~\cite{renaud2025convergence} on the fine-level denoiser $\boldD_{\sigma}$, the data-fidelity term $f$, the lower boundedness of $h_{\sigma}$, and using a suitable step size $\delta$ at the fine-level update in Algorithm~\ref{algo:wml_spnp}, the tail sequence inherits the corresponding stationarity guarantees. These guarantees do not establish convergence of the \gls{ML} updates themselves or descent of the fine-level objective during that phase.

\section{Experiments and Results}\label{sec:experiments}

\begin{figure*}[!t]
	\centering
	\setlength{\tabcolsep}{1pt}
	\renewcommand{\arraystretch}{0}
	
	\begin{tabular}{@{}c@{\hspace{1pt}}c@{}}
		\adjustbox{valign=c}{\rotatebox[origin=c]{90}{\textbf{50 views}}}
		&
		\adjustbox{valign=c}{\includegraphics[width=0.985\textwidth]{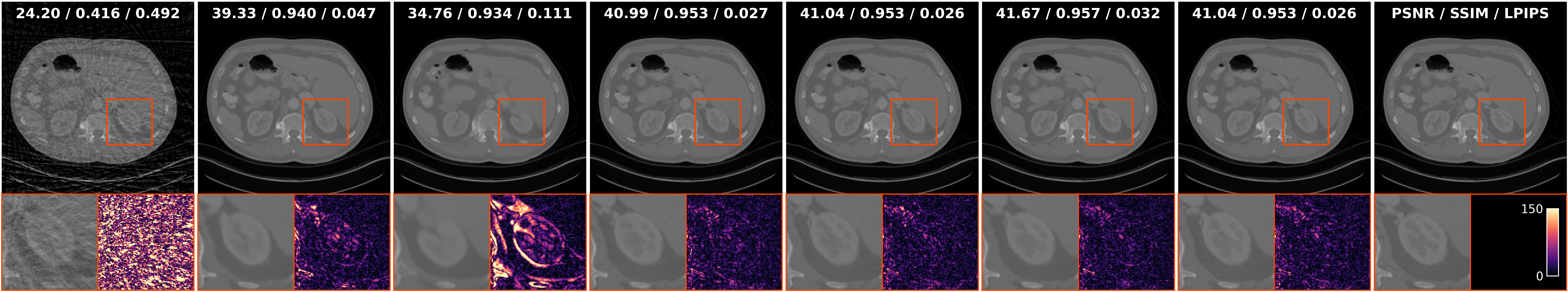}}
		\\[0mm]
		\adjustbox{valign=c}{\rotatebox[origin=c]{90}{\textbf{40 views}}}
		&
		\adjustbox{valign=c}{\includegraphics[width=0.985\textwidth]{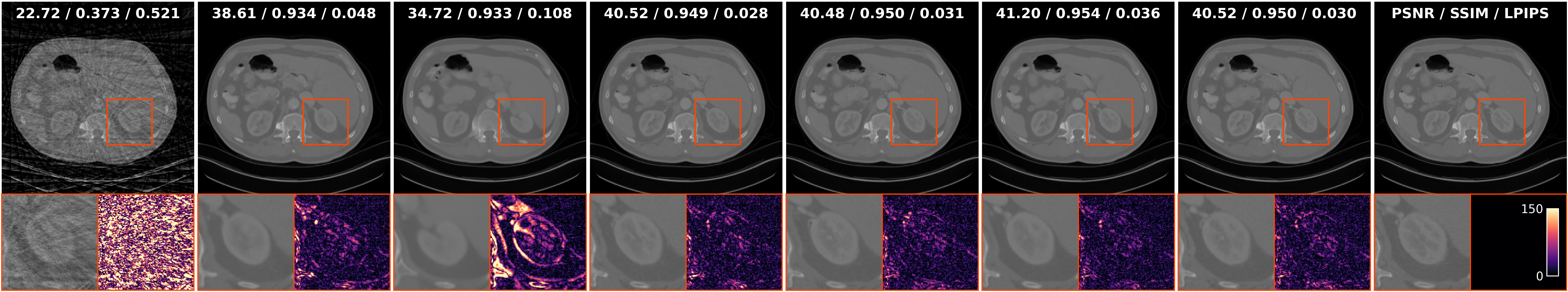}}
		\\[0mm]
		\adjustbox{valign=c}{\rotatebox[origin=c]{90}{\textbf{30 views}}}
		&
		\adjustbox{valign=c}{\includegraphics[width=0.985\textwidth]{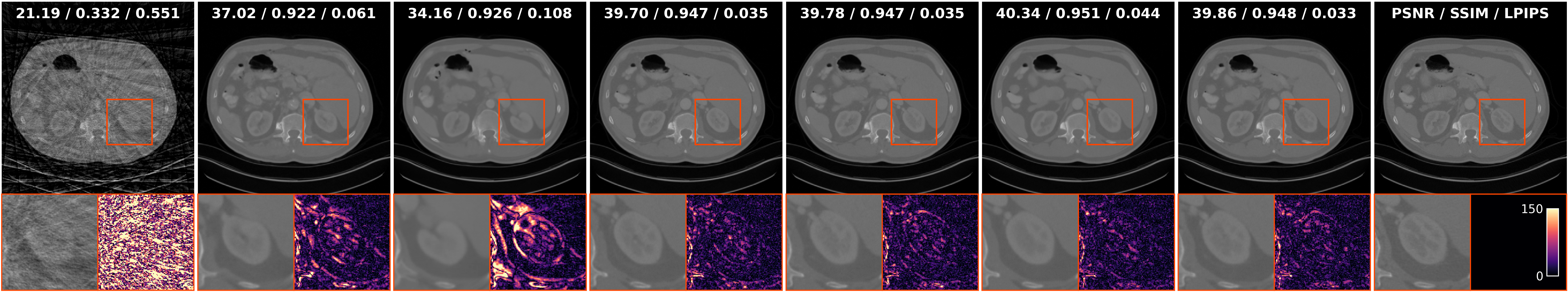}}
		\\[0mm]
		\adjustbox{valign=c}{\rotatebox[origin=c]{90}{\textbf{20 views}}}
		&
		\adjustbox{valign=c}{\ImageWithMethodLabels[0.985\textwidth]{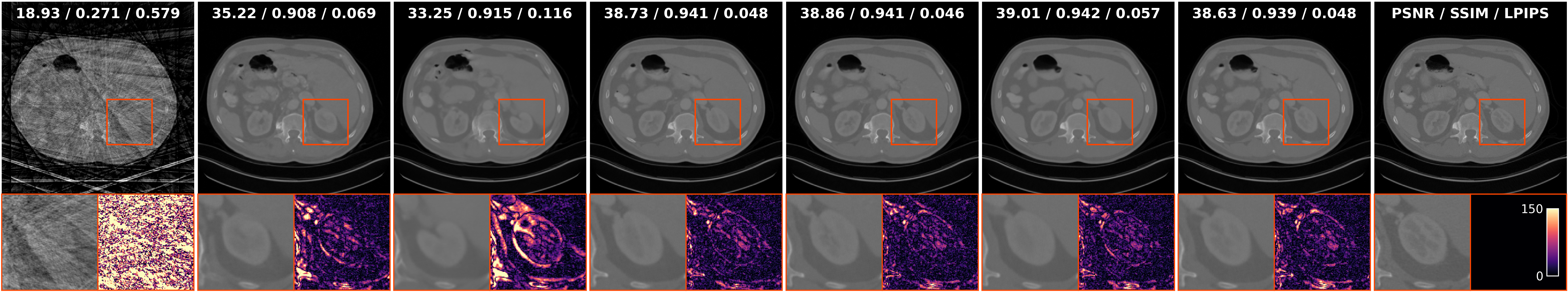}}
	\end{tabular}
	
	\caption{
		Qualitative comparison of sparse-view \gls{CT} reconstruction with $n_{\theta}\in\{20,30,40,50\}$ projection views on the Lymph Nodes dataset. For each method, the top image shows the full reconstruction displayed using a $[-1000,1500]$~HU window, while the bottom panels show the magnified region and its corresponding error map. PSNR/SSIM/LPIPS values are reported above each reconstruction.
	}
	\label{fig:svct_20_40_qualitative}
\end{figure*}

Our experiments were carried out on \gls{2D} images, i.e., with \gls{ML} spaces $\calX_\ell=\R^{n/4^\ell}$, $n=512^2$.

\subsection{Experimental Settings}

\subsubsection{Datasets and Preprocessing}

We considered two datasets for the \gls{SVCT} reconstruction experiments: the mediastinal and abdominal \gls{CT} Lymph Nodes dataset~\cite{roth2014new} and the head \gls{CT} CQ500 dataset~\cite{chilamkurthy2018deep}. The corresponding training sets contain 140 patients and 87,110 axial slices for the Lymph Nodes dataset, and 45 patients and 11,907 axial slices for CQ500, respectively. The validation sets are composed of 6,347 axial slices from 10 held-out patients and 1,280 axial slices from five held-out patients for the Lymph Nodes and CQ500 datasets, respectively. For testing, we reported results on 200 axial slices extracted from 10 held-out patients for each dataset. For CQ500, we retained only axial slices spanning approximately from the inferior mental/mandibular region to the superior frontal calvarial region, in order to focus the evaluation on anatomically relevant head \gls{CT} sections. The same test patients and slices were used for all sparse-view settings. All images have an initial spatial resolution of 512\texttimes512 pixels. The images are first clipped in \gls{HU}, respectively using the ranges $[-1000,1500]$ and $[-1000,2000]$ for the two datasets. They are then normalized to $[0,1]$ for training the denoising and diffusion models.

\subsubsection{Sparse-view CT Acquisition Simulation}

To simulate realistic \gls{CT} measurements, the clipped \gls{HU} images were first converted into linear attenuation coefficients with values ranging in $[0,0.04825]~\mathrm{mm}^{-1}$ and in  $[0,0.0579]~\mathrm{mm}^{-1}$, respectively. Projection data were generated from the linear attenuation images using the CTorch projector~\cite{jiang2025ctorch}, with a fan-beam geometry and the separable-footprint projection algorithm. We used $n_\rmd=1024$ detector elements with a detector spacing of $1~\mathrm{mm}$. Following~\cite{erdogan2002monotonic}, we simulated raw measurements $\tilde{\boldy}$ with a monochromatic source, a photon intensity $I_0 = 10^6$ and Poisson noise (no electronic noise). The raw measurements $\tilde{\boldy}$ were then log-transformed to obtain $\boldy$ (cf.~\eqref{eq:forward_pb}), and the data-fidelity distance function $d$ was defined as a weighted squared $\ell_2$-norm, with weights given by the raw photon counts $\tilde{\boldy}$ to approximate the Poisson negative log-likelihood \cite{elbakri2002statistical}. We considered sparse-view \gls{CT} settings with $n_\theta \in \{20, 30, 40, 50\}$ projection views. The implementation was based on the DeepInverse library \cite{tachella2025deepinverse}.  

\subsubsection{Methods for Comparison}

\Gls{FBP} was included as a classical analytical reconstruction baseline, while \gls{DDS}~\cite{chung2024decomposed} and \gls{DiffPIR}~\cite{zhu2023denoising} were considered as diffusion-based inverse-problem solvers. The grid searches described below were conducted on a validation subset of 100 slices from 10 patients, using the mean \gls{PSNR} as the selection criterion. Similar validation behavior was observed on both datasets, and the same selected parameter settings were used for both.

\Gls{DDS} was run with $200$ sampling steps. The number of conjugate-gradient iterations per data-consistency update was selected from $n_{\mathrm{cg}} \in \{5,10,15,20\}$, yielding $n_{\mathrm{cg}}=10$, for all view settings. \Gls{DiffPIR} was run with $100$ \gls{DDIM} sampling steps and $50$ iterations per proximal update. For each acquisition setting, its parameter $\lambda$ was selected from $\{0.001,0.005,0.01,0.05,0.1,0.5\}$, yielding $\lambda=0.01$, $0.05$, $0.1$, and $0.1$ for $n_\theta=20$, $30$, $40$, and $50$, respectively.

We also compared against stochastic \gls{PnP} reconstruction baselines, namely Prox-\gls{SNORE}~\cite{renaud2025convergence} and Annealed Prox-\gls{SNORE}~\cite{renaud2024plug}. Both methods are implemented using an equivariant DRUNet denoiser~\cite{zhang2021plug,terris2024equivariant}, trained with Gaussian noise levels sampled uniformly from $[0,0.2]$. The step size was chosen empirically and fixed at $\delta=10^{-4}$, while the regularization strength was set to $\lambda=0.99\,\sigma^2/\delta$. For each number of projection views $n_\theta$, the denoiser noise level was selected from the grid $\sigma \in \{0.01,0.0125,\ldots,0.0275, 0.03\}$. To reduce the computational cost of this grid search, each candidate was evaluated using the annealed variant with an iteration budget of $K=300$, where $\sigma$ denotes the final noise level of the annealing schedule. The selected reconstruction parameters were also used for the standard Prox-\gls{SNORE} variant and are reported in Table~\ref{tab:wmlspnp_hyperparams} in the reconstruction part. For the annealed variant, annealing was applied only during the first 40 iterations. The noise schedule was divided into five uniform plateaus, each lasting eight iterations, and decreased from 0.2 to the final noise level $\sigma$ reported in Table~\ref{tab:wmlspnp_hyperparams}. Finally, we compared against \gls{ML}-\gls{PnP}~\cite{laurent2025multilevel}. Since the original implementation relies on an explicit gradient step for the data-fidelity guidance, we observed numerical instabilities in the sparse-view \gls{CT} setting considered here, similar to those reported in~\cite{denker2026stability}. To obtain a more stable and comparable baseline, we extended \gls{ML}-\gls{PnP} by replacing this explicit data-fidelity step with a proximal data-consistency update that embeds the coherence term, leading to the Prox-\gls{ML}-\gls{PnP} variant used in our experiments. The denoising noise level $\sigma$ and the regularization strength $\lambda$ were chosen consistently as indicated in Table~\ref{tab:wmlspnp_hyperparams}. We used the same \gls{ML} parameters as in Table~\ref{tab:wmlspnp_hyperparams} too.

All methods were carefully tuned to ensure a fair comparison. In particular, the hyperparameters of the SNORE-based methods and ML-SPnP, as mentioned previously, were selected such that the algorithms converge to stable solutions with comparable balances between data fidelity and regularization. All inference experiments, including those for our methods, were performed on the same NVIDIA RTX A5000 GPU with 24 GB of memory, under identical hardware conditions.

\begin{figure*}[!t]
	\centering
	\setlength{\tabcolsep}{1pt}
	\renewcommand{\arraystretch}{0}
	
	\begin{tabular}{@{}c@{\hspace{1pt}}c@{}}
	
		\adjustbox{valign=c}{\rotatebox[origin=c]{90}{\textbf{50 views}}}
		&
		\adjustbox{valign=c}{\includegraphics[width=0.985\textwidth]{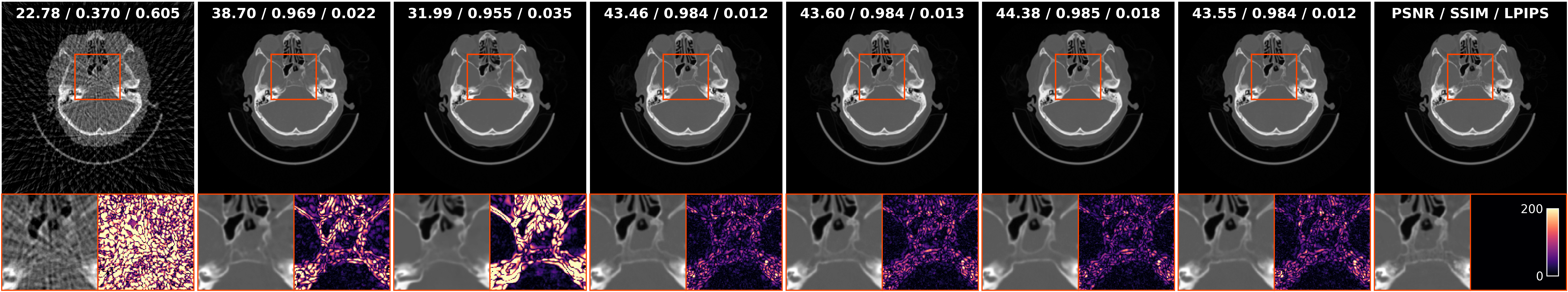}}
		\\[0mm]
		\adjustbox{valign=c}{\rotatebox[origin=c]{90}{\textbf{40 views}}}
		&
		\adjustbox{valign=c}{\includegraphics[width=0.985\textwidth]{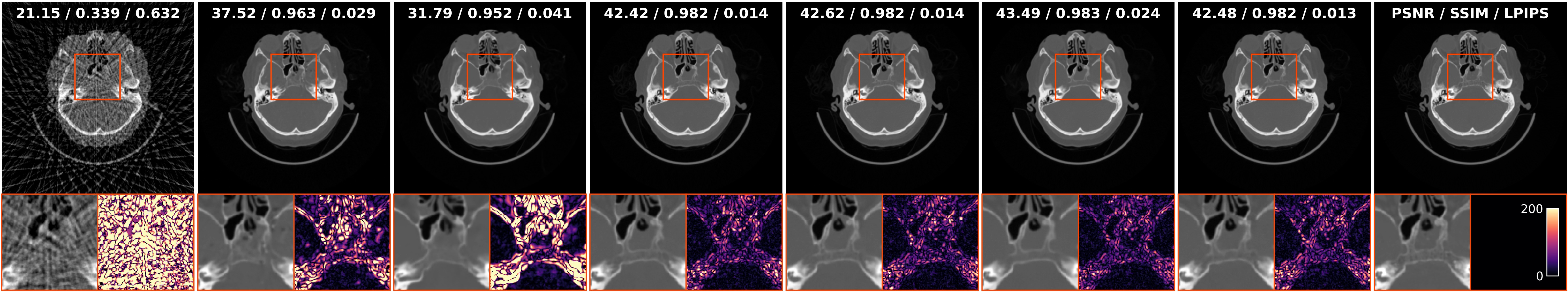}}
		\\[0mm]
		\adjustbox{valign=c}{\rotatebox[origin=c]{90}{\textbf{30 views}}}
		&
		\adjustbox{valign=c}{\includegraphics[width=0.985\textwidth]{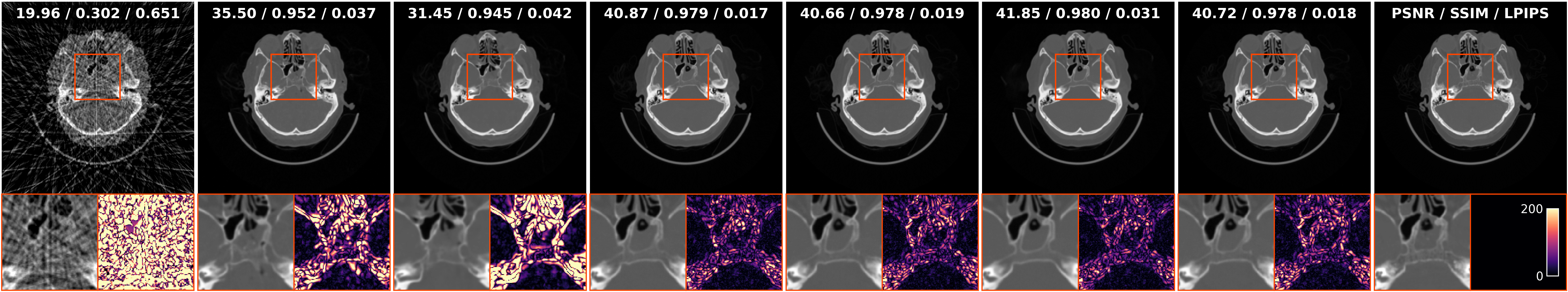}}
		\\[0mm]
		\adjustbox{valign=c}{\rotatebox[origin=c]{90}{\textbf{20 views}}}
		&
		\adjustbox{valign=c}{\ImageWithMethodLabels[0.985\textwidth]{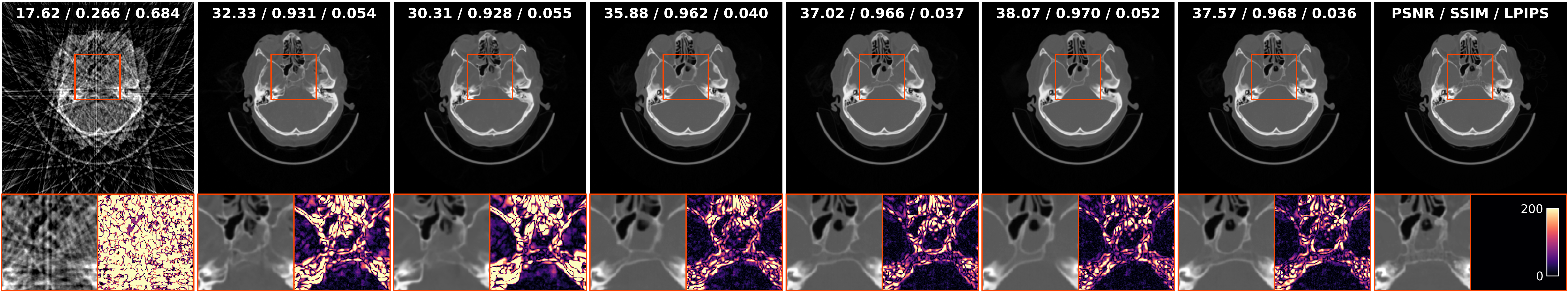}}
	
	\end{tabular}
	
	\caption{
		Qualitative comparison of sparse-view \gls{CT} reconstruction with $n_{\theta}\in\{20,30,40,50\}$ projection views on the CQ500 head \gls{CT} dataset. For each method, the top image shows the full reconstruction displayed using a $[-1000,2000]$~HU window, while the bottom panels show the magnified region and its corresponding error map. PSNR/SSIM/LPIPS values are reported above each reconstruction.
	}
	\label{fig:svct_cq500_qualitative}
\end{figure*}

\subsubsection{Implementation details of ML-SPnP}

We used an equivariant DRUNet as the denoising backbone of ML-SPnP, trained with Gaussian noise levels $\sigma$ sampled uniformly from $[0,0.2]$. We trained three resolution-specific denoisers, operating at resolutions $(512/2^\ell) \times (512/2^\ell)$ for $\ell \in \{0,1,2\}$. For $\ell=0,1,2$, training to convergence required approximately $120$, $48$, and $24$ GPU-hours, with batch sizes $8$, $64$, and $128$ and model sizes $255.60$, $130.55$, and $130.55$~MB, respectively. This amounts to $192$ GPU-hours and $516.70$ MB per model triplet, with a peak training GPU memory usage of $44.3$ GB.
To avoid repeatedly applying the full-resolution forward operator at coarse scales, we use $\boldA_\ell \approx \boldA\bar{\boldP}_{\ell}$ as a resolution-dependent discretization of the X-ray transform. This operator was normalized by the wavelet scaling factor $1/2^\ell$, which compensates for the intensity rescaling of the approximation subband across decomposition levels. Moreover, $\tau_{k, l}$ is set to one following \cite{laurent2025multilevel}. The reconstruction hyperparameters used in our experiments are reported in Table~\ref{tab:wmlspnp_hyperparams}.

\begin{table}[H]
	\centering
	\footnotesize
	\renewcommand{\arraystretch}{1.25}
	\setlength{\tabcolsep}{5pt}
	\begin{tabular}{@{}c c c c c c c@{}}
		\toprule
		\multirow{2}{*}{$n_\theta$}
		& \multicolumn{3}{c}{Reconstruction}
		& \multicolumn{3}{c}{ML parameters} \\
		\cmidrule(lr){2-4}\cmidrule(l){5-7}
		& $\sigma$ & $\lambda$ & $\delta$
		& $K_\ell$ & $K_{\mathrm{ML}}$ & $L$ \\
		\midrule
		$50$ & $0.015$
		& \multirow{4}{*}{$0.99\,\dfrac{\sigma^2}{\delta}$}
		& \multirow{4}{*}{$10^{-4}$}
		& $3$ & $4$ & \multirow{4}{*}{$3$} \\
		$40$ & $0.015$  & & & $3$ & $6$ & \\
		$30$ & $0.0175$ & & & $5$ & $4$ & \\
		$20$ & $0.025$  & & & $5$ & $6$ & \\
		\bottomrule
	\end{tabular}
	\caption{ML-SPnP hyperparameters and \gls{ML} schedules.}
	\label{tab:wmlspnp_hyperparams}
\end{table}
\begin{table*}[t]
	\centering
	\setlength{\tabcolsep}{2.6pt}
	\renewcommand{\arraystretch}{1.35}
	\def\meanci#1#2#3#4{%
	  \begingroup
	  \renewcommand{\arraystretch}{1}%
	  \begin{tabular}[t]{@{}c@{}}%
	    #1\,{\normalfont\scriptsize\color{black!80}$\pm\,#2$}\\[-1pt]
	    {\normalfont\scriptsize\color{black!80}$[#3,\,#4]$}%
	  \end{tabular}%
	  \endgroup
	}
	\makeatletter
	\def\datasetlabel#1{%
	  \begingroup
	  \setbox0=\hbox{\meanci{0.00}{0.00}{0.00}{0.00}}%
	  \dimen0=\dimexpr\dp0-\dp\@arstrutbox\relax
	  \ifdim\dimen0<0pt \dimen0=0pt\fi
	  \multiply\dimen0 by -7
	  \divide\dimen0 by 2
	  \edef\datasetoffset{\the\dimen0}%
	  \multirow{7}{*}[\datasetoffset]{%
	    \rotatebox[origin=c]{90}{\textbf{#1}}}%
	  \endgroup
	}
	\makeatother
	\resizebox{\textwidth}{!}{%
	\begin{tabular}{cl*{12}{c}}
		\toprule[1.2pt]
		\textbf{Dataset} & \textbf{Method}
		& \multicolumn{3}{c}{\textbf{20 views}}
		& \multicolumn{3}{c}{\textbf{30 views}}
		& \multicolumn{3}{c}{\textbf{40 views}}
		& \multicolumn{3}{c}{\textbf{50 views}} \\
		\cmidrule(lr){3-5}
		\cmidrule(lr){6-8}
		\cmidrule(lr){9-11}
		\cmidrule(lr){12-14}
		& & PSNR $\uparrow$ & SSIM $\uparrow$ & LPIPS $\downarrow$
		& PSNR $\uparrow$ & SSIM $\uparrow$ & LPIPS $\downarrow$
		& PSNR $\uparrow$ & SSIM $\uparrow$ & LPIPS $\downarrow$
		& PSNR $\uparrow$ & SSIM $\uparrow$ & LPIPS $\downarrow$ \\
		\midrule
		
		\datasetlabel{Lymph Node}
		& FBP
		& \meanci{18.42}{0.43}{18.23}{18.62} & \meanci{0.277}{0.014}{0.265}{0.287} & \meanci{0.609}{0.012}{0.600}{0.617}
		& \meanci{20.76}{0.54}{20.53}{21.04} & \meanci{0.330}{0.012}{0.320}{0.338} & \meanci{0.574}{0.014}{0.564}{0.585}
		& \meanci{22.33}{0.53}{22.10}{22.58} & \meanci{0.370}{0.010}{0.362}{0.377} & \meanci{0.545}{0.012}{0.536}{0.554}
		& \meanci{23.71}{0.61}{23.43}{24.00} & \meanci{0.409}{0.010}{0.402}{0.415} & \meanci{0.515}{0.013}{0.506}{0.524} \\
		& DiffPIR
		& \meanci{35.86}{0.99}{35.36}{36.52} & \meanci{0.924}{0.014}{0.913}{0.934} & \meanci{0.063}{0.010}{0.057}{0.071}
		& \meanci{38.05}{0.90}{37.52}{38.67} & \meanci{0.937}{0.011}{0.928}{0.945} & \meanci{0.048}{0.009}{0.043}{0.054}
		& \meanci{39.35}{0.86}{38.78}{39.95} & \meanci{0.945}{0.009}{0.937}{0.952} & \meanci{0.041}{0.007}{0.036}{0.046}
		& \meanci{40.19}{0.84}{39.60}{40.78} & \meanci{0.951}{0.009}{0.943}{0.957} & \meanci{0.038}{0.006}{0.033}{0.042} \\
		& DDS
		& \meanci{34.19}{0.80}{33.60}{34.76} & \meanci{0.930}{0.013}{0.922}{0.940} & \meanci{0.086}{0.019}{0.073}{0.102}
		& \meanci{35.51}{0.84}{34.88}{36.14} & \meanci{0.943}{0.010}{0.935}{0.950} & \meanci{0.078}{0.019}{0.064}{0.094}
		& \meanci{36.12}{0.88}{35.39}{36.81} & \meanci{0.949}{0.009}{0.941}{0.955} & \meanci{0.075}{0.020}{0.061}{0.092}
		& \meanci{36.46}{0.93}{35.68}{37.20} & \meanci{0.952}{0.009}{0.944}{0.958} & \meanci{0.075}{0.021}{0.060}{0.091} \\
		& Prox-SNORE
		& \meanci{38.10}{1.33}{37.13}{39.15} & \meanci{0.947}{0.007}{0.941}{0.954} & \meanci{\underline{0.046}}{0.007}{0.041}{0.054}\cellcolor{blue!15}
		& \meanci{40.21}{1.23}{39.05}{41.06} & \meanci{0.956}{0.008}{0.948}{0.962} & \meanci{0.034}{0.010}{0.027}{0.044}
		& \meanci{40.85}{1.50}{39.35}{41.87} & \meanci{0.957}{0.011}{0.946}{0.965} & \meanci{\underline{0.033}}{0.017}{0.022}{0.052}\cellcolor{blue!15}
		& \meanci{41.61}{1.33}{40.28}{42.57} & \meanci{0.960}{0.009}{0.951}{0.967} & \meanci{\underline{0.029}}{0.014}{0.020}{0.044}\cellcolor{blue!15} \\
		& ann-Prox-SNORE
		& \meanci{38.67}{0.74}{38.12}{39.31} & \meanci{\underline{0.949}}{0.007}{0.943}{0.954}\cellcolor{blue!15} & \meanci{\textbf{0.045}}{0.008}{0.040}{0.052}\cellcolor{green!20}
		& \meanci{\underline{40.58}}{0.67}{40.06}{41.08}\cellcolor{blue!15} & \meanci{\underline{0.957}}{0.006}{0.951}{0.961}\cellcolor{blue!15} & \meanci{\textbf{0.031}}{0.006}{0.027}{0.038}\cellcolor{green!20}
		& \meanci{\underline{41.44}}{0.70}{40.86}{41.98}\cellcolor{blue!15} & \meanci{\underline{0.960}}{0.006}{0.954}{0.965}\cellcolor{blue!15} & \meanci{\textbf{0.026}}{0.007}{0.022}{0.032}\cellcolor{green!20}
		& \meanci{\underline{42.05}}{0.74}{41.40}{42.62}\cellcolor{blue!15} & \meanci{\underline{0.962}}{0.006}{0.956}{0.967}\cellcolor{blue!15} & \meanci{\textbf{0.024}}{0.008}{0.020}{0.031}\cellcolor{green!20} \\
		& Prox-ML-PnP
		& \meanci{\textbf{39.15}}{0.91}{38.51}{39.81}\cellcolor{green!20} & \meanci{\textbf{0.950}}{0.008}{0.944}{0.956}\cellcolor{green!20} & \meanci{0.058}{0.011}{0.051}{0.067}
		& \meanci{\textbf{41.10}}{0.79}{40.44}{41.69}\cellcolor{green!20} & \meanci{\textbf{0.960}}{0.006}{0.955}{0.964}\cellcolor{green!20} & \meanci{0.041}{0.008}{0.035}{0.048}
		& \meanci{\textbf{42.18}}{0.69}{41.59}{42.71}\cellcolor{green!20} & \meanci{\textbf{0.964}}{0.006}{0.958}{0.968}\cellcolor{green!20} & \meanci{\underline{0.033}}{0.008}{0.028}{0.041}\cellcolor{blue!15}
		& \meanci{\textbf{42.76}}{0.72}{42.13}{43.32}\cellcolor{green!20} & \meanci{\textbf{0.965}}{0.006}{0.960}{0.970}\cellcolor{green!20} & \meanci{0.031}{0.008}{0.025}{0.038} \\
		& ML-SPnP
		& \meanci{\underline{38.76}}{1.06}{37.88}{39.53}\cellcolor{blue!15} & \meanci{\underline{0.949}}{0.008}{0.942}{0.955}\cellcolor{blue!15} & \meanci{\textbf{0.045}}{0.009}{0.039}{0.053}\cellcolor{green!20}
		& \meanci{40.45}{0.90}{39.66}{41.12} & \meanci{\underline{0.957}}{0.006}{0.950}{0.962}\cellcolor{blue!15} & \meanci{\underline{0.032}}{0.007}{0.027}{0.038}\cellcolor{blue!15}
		& \meanci{\underline{41.44}}{0.70}{40.86}{41.98}\cellcolor{blue!15} & \meanci{\underline{0.960}}{0.006}{0.954}{0.965}\cellcolor{blue!15} & \meanci{\textbf{0.026}}{0.007}{0.022}{0.032}\cellcolor{green!20}
		& \meanci{\underline{42.05}}{0.74}{41.41}{42.62}\cellcolor{blue!15} & \meanci{\underline{0.962}}{0.006}{0.956}{0.967}\cellcolor{blue!15} & \meanci{\textbf{0.024}}{0.008}{0.020}{0.030}\cellcolor{green!20} \\
		
		\midrule
		
		\datasetlabel{CQ500}
		& FBP
		& \meanci{16.43}{0.67}{16.13}{16.79} & \meanci{0.246}{0.013}{0.239}{0.253} & \meanci{0.683}{0.012}{0.675}{0.690}
		& \meanci{18.89}{0.61}{18.61}{19.19} & \meanci{0.289}{0.012}{0.283}{0.296} & \meanci{0.660}{0.011}{0.653}{0.666}
		& \meanci{20.47}{0.49}{20.20}{20.74} & \meanci{0.328}{0.012}{0.321}{0.334} & \meanci{0.633}{0.012}{0.626}{0.639}
		& \meanci{21.84}{0.56}{21.60}{22.12} & \meanci{0.360}{0.011}{0.354}{0.366} & \meanci{0.606}{0.013}{0.598}{0.612} \\
		& DiffPIR
		& \meanci{32.05}{0.98}{31.52}{32.47} & \meanci{0.922}{0.014}{0.914}{0.929} & \meanci{0.056}{0.013}{0.050}{0.064}
		& \meanci{35.35}{1.04}{34.77}{35.83} & \meanci{0.946}{0.009}{0.940}{0.951} & \meanci{0.036}{0.008}{0.032}{0.042}
		& \meanci{37.39}{1.10}{36.77}{37.92} & \meanci{0.959}{0.007}{0.954}{0.963} & \meanci{0.026}{0.006}{0.023}{0.030}
		& \meanci{38.73}{1.10}{38.12}{39.29} & \meanci{0.967}{0.006}{0.963}{0.970} & \meanci{0.020}{0.005}{0.017}{0.023} \\
		& DDS
		& \meanci{28.90}{1.56}{28.11}{29.72} & \meanci{0.911}{0.018}{0.901}{0.919} & \meanci{0.068}{0.018}{0.060}{0.079}
		& \meanci{30.08}{1.85}{29.14}{31.09} & \meanci{0.929}{0.015}{0.920}{0.936} & \meanci{0.054}{0.015}{0.047}{0.063}
		& \meanci{30.53}{1.96}{29.53}{31.61} & \meanci{0.937}{0.014}{0.929}{0.945} & \meanci{0.049}{0.013}{0.042}{0.057}
		& \meanci{30.74}{2.04}{29.69}{31.87} & \meanci{0.941}{0.014}{0.934}{0.948} & \meanci{0.046}{0.013}{0.040}{0.054} \\
		& Prox-SNORE
		& \meanci{35.66}{1.42}{34.85}{36.32} & \meanci{0.956}{0.010}{0.950}{0.961} & \meanci{0.038}{0.013}{0.032}{0.047}
		& \meanci{39.93}{1.26}{39.20}{40.50} & \meanci{0.974}{0.005}{0.971}{0.977} & \meanci{\underline{0.020}}{0.007}{0.017}{0.024}\cellcolor{blue!15}
		& \meanci{41.68}{1.36}{40.89}{42.31} & \meanci{0.979}{0.004}{0.976}{0.981} & \meanci{\underline{0.015}}{0.006}{0.012}{0.019}\cellcolor{blue!15}
		& \meanci{42.88}{1.17}{42.20}{43.44} & \meanci{\underline{0.982}}{0.003}{0.980}{0.983}\cellcolor{blue!15} & \meanci{\textbf{0.012}}{0.004}{0.011}{0.015}\cellcolor{green!20} \\
		& ann-Prox-SNORE
		& \meanci{37.03}{0.80}{36.55}{37.42} & \meanci{0.964}{0.006}{0.960}{0.968} & \meanci{\underline{0.029}}{0.008}{0.024}{0.035}\cellcolor{blue!15}
		& \meanci{\underline{40.27}}{0.87}{39.76}{40.71}\cellcolor{blue!15} & \meanci{\underline{0.976}}{0.004}{0.973}{0.978}\cellcolor{blue!15} & \meanci{\textbf{0.018}}{0.005}{0.015}{0.021}\cellcolor{green!20}
		& \meanci{\underline{41.96}}{0.94}{41.42}{42.44}\cellcolor{blue!15} & \meanci{\underline{0.980}}{0.003}{0.978}{0.981}\cellcolor{blue!15} & \meanci{\textbf{0.013}}{0.003}{0.012}{0.015}\cellcolor{green!20}
		& \meanci{\underline{42.99}}{0.91}{42.47}{43.47}\cellcolor{blue!15} & \meanci{\underline{0.982}}{0.002}{0.980}{0.983}\cellcolor{blue!15} & \meanci{\textbf{0.012}}{0.002}{0.010}{0.013}\cellcolor{green!20} \\
		& Prox-ML-PnP
		& \meanci{\textbf{37.75}}{0.79}{37.26}{38.14}\cellcolor{green!20} & \meanci{\textbf{0.967}}{0.005}{0.964}{0.971}\cellcolor{green!20} & \meanci{0.036}{0.012}{0.029}{0.044}
		& \meanci{\textbf{41.13}}{0.88}{40.60}{41.57}\cellcolor{green!20} & \meanci{\textbf{0.978}}{0.004}{0.975}{0.980}\cellcolor{green!20} & \meanci{0.023}{0.008}{0.019}{0.029}
		& \meanci{\textbf{42.84}}{0.95}{42.28}{43.33}\cellcolor{green!20} & \meanci{\textbf{0.982}}{0.003}{0.979}{0.983}\cellcolor{green!20} & \meanci{0.017}{0.006}{0.014}{0.021}
		& \meanci{\textbf{43.89}}{0.92}{43.35}{44.37}\cellcolor{green!20} & \meanci{\textbf{0.983}}{0.002}{0.982}{0.985}\cellcolor{green!20} & \meanci{\underline{0.014}}{0.004}{0.012}{0.017}\cellcolor{blue!15} \\
		& ML-SPnP
		& \meanci{\underline{37.29}}{0.75}{36.84}{37.66}\cellcolor{blue!15} & \meanci{\underline{0.966}}{0.005}{0.962}{0.969}\cellcolor{blue!15} & \meanci{\textbf{0.027}}{0.008}{0.023}{0.033}\cellcolor{green!20}
		& \meanci{40.23}{0.90}{39.71}{40.68} & \meanci{0.975}{0.004}{0.973}{0.978} & \meanci{\textbf{0.018}}{0.005}{0.015}{0.021}\cellcolor{green!20}
		& \meanci{\underline{41.96}}{0.96}{41.40}{42.44}\cellcolor{blue!15} & \meanci{\underline{0.980}}{0.003}{0.978}{0.981}\cellcolor{blue!15} & \meanci{\textbf{0.013}}{0.003}{0.012}{0.015}\cellcolor{green!20}
		& \meanci{\underline{42.99}}{0.92}{42.47}{43.47}\cellcolor{blue!15} & \meanci{\underline{0.982}}{0.002}{0.980}{0.983}\cellcolor{blue!15} & \meanci{\textbf{0.012}}{0.002}{0.010}{0.013}\cellcolor{green!20} \\
		
		\bottomrule[1.2pt]
	\end{tabular}%
	}
	\caption{
		Reconstruction quality for sparse-view \gls{CT} with $n_{\theta} \in \{20, 30, 40, 50\}$ projection views on the Lymph Nodes and CQ500 datasets. Entries report the mean and patient-level \gls{SD} for \gls{PSNR}, \gls{SSIM}, and \gls{LPIPS}, with 95\% confidence intervals shown below. For each dataset, view count, and metric, the best and second-best mean values across all methods except \gls{FBP} are shown in bold and underlined, respectively. \colorbox{green!20}{Green} and \colorbox{blue!15}{blue} highlights indicate the best and second-best mean values, respectively, among the optimization-based methods (Prox-SNORE, ann-Prox-SNORE, Prox-ML-PnP, and ML-SPnP).	Ties at the reported precision receive the same formatting.
		}
	\label{tab:quantitative_comparison_template}
\end{table*}

ML-SPnP was configured to perform \gls{ML} updates during the first $K_{\mathrm{ML}}$ outer iterations. The configurations reported in the \gls{ML} part of Table~\ref{tab:wmlspnp_hyperparams} were selected by a grid search over $K_\ell\in\{3,5,7\}$, $K_{\mathrm{ML}}\in\{2,4,6,8\}$, and $L\in\{2,3\}$, selecting the configuration giving the best \gls{PSNR} after $10$ seconds of execution.

\subsubsection{Evaluation metrics}

Reconstruction quality was evaluated using \gls{PSNR}, \gls{SSIM}, and \gls{LPIPS} \cite{zhang2018unreasonable}. We also reported reconstruction time to compare computational efficiency. For each optimization-based method, dataset, and acquisition setting, we selected a fixed iteration budget based on the mean validation \gls{PSNR} improvement per second over a trailing window, and used it at test time. Let $Q_{\rm val}(K)$ and $T_{\rm val}(K)$ denote the mean validation \gls{PSNR} and runtime in seconds after $K$ optimization iterations, respectively. For the ordered candidate budgets $\mathcal{K} = \{K_1 < \cdots < K_m\}$ shared by all validation runs, we selected
{\small
\begin{equation}
K^\star = \min\left\{
K_j \in \mathcal{K}: j>w,\;
\frac{Q_{\rm val}(K_j)-Q_{\rm val}(K_{j-w})}
     {T_{\rm val}(K_j)-T_{\rm val}(K_{j-w})}
<\tau
\right\}
\nonumber
\end{equation}
}where $\tau=0.01$ and $w=10$ is the window length. The ratio estimates the validation \gls{PSNR} derivative with respect to runtime, with near-zero values indicating a plateau, while the window reduces sensitivity to stochastic fluctuations.

\subsection{Results}
\begin{table*}[t]
	\centering
	\setlength{\tabcolsep}{3.5pt}
	\renewcommand{\arraystretch}{1.15}
	\resizebox{\textwidth}{!}{%
	\begin{tabular}{cl*{12}{c}}
		\toprule[1.2pt]
		\textbf{Dataset} & \textbf{Method}
		& \multicolumn{3}{c}{\textbf{20 views}}
		& \multicolumn{3}{c}{\textbf{30 views}}
		& \multicolumn{3}{c}{\textbf{40 views}}
		& \multicolumn{3}{c}{\textbf{50 views}} \\
		\cmidrule(lr){3-5}
		\cmidrule(lr){6-8}
		\cmidrule(lr){9-11}
		\cmidrule(lr){12-14}
		& & Time (s) $\downarrow$ & NFE $\downarrow$ & $K^{\ast}$
		& Time (s) $\downarrow$ & NFE $\downarrow$ & $K^{\ast}$
		& Time (s) $\downarrow$ & NFE $\downarrow$ & $K^{\ast}$
		& Time (s) $\downarrow$ & NFE $\downarrow$ & $K^{\ast}$ \\
		\midrule
		
		\multirow{7}{*}{\rotatebox[origin=c]{90}{\textbf{Lymph Node}}}
		& FBP
		& 0.00 & -- & --
		& 0.00 & -- & --
		& 0.00 & -- & --
		& 0.00 & -- & -- \\
		& DiffPIR
		& \textbf{16.01} & 100 & 100
		& \underline{19.45} & 100 & 100
		& 23.11 & 100 & 100
		& 27.86 & 100 & 100 \\
		& DDS
		& \underline{16.57} & 200 & 200
		& \textbf{18.18} & 200 & 200
		& 19.91 & 200 & 200
		& 22.41 & 200 & 200 \\
		& Prox-SNORE
		& 85.45 & 437 & 437
		& 60.33 & 259 & 259
		& 49.82 & 186 & 186
		& 42.92 & 140 & 140 \\
		& ann-Prox-SNORE
		& 42.45 & 217 & 217
		& 31.69 & 136 & 136
		& 20.98 & 78 & 78
		& 22.46 & 73 & 73 \\
		& Prox-ML-PnP
		& \cellcolor{green!20}25.30 & 198 & 104
		& \cellcolor{green!20}24.82 & 156 & 90
		& \cellcolor{green!20}\textbf{18.16} & 118 & 52
		& \cellcolor{green!20}\textbf{16.18} & 88 & 42 \\
		& ML-SPnP
		& \cellcolor{blue!15}29.37 & 198 & 128
		& \cellcolor{blue!15}30.02 & 164 & 114
		& \cellcolor{blue!15}\underline{18.59} & 98 & 56
		& \cellcolor{blue!15}\underline{17.88} & 79 & 49 \\
		
		\midrule
		
		\multirow{7}{*}{\rotatebox[origin=c]{90}{\textbf{CQ500}}}
		& FBP
		& 0.00 & 0 & --
		& 0.00 & 0 & --
		& 0.00 & 0 & --
		& 0.00 & 0 & -- \\
		& DiffPIR
		& \textbf{15.96} & 100 & 100
		& \underline{19.72} & 100 & 100
		& \underline{23.39} & 100 & 100
		& \underline{26.61} & 100 & 100 \\
		& DDS
		& \underline{16.59} & 200 & 200
		& \textbf{18.27} & 200 & 200
		& \textbf{20.00} & 200 & 200
		& \textbf{21.50} & 200 & 200 \\
		& Prox-SNORE
		& 97.44 & 500 & 500
		& 115.73 & 500 & 500
		& 95.75 & 356 & 356
		& 73.18 & 236 & 236 \\
		& ann-Prox-SNORE
		& 67.02 & 342 & 342
		& 41.73 & 179 & 179
		& \cellcolor{blue!15}34.88 & 130 & 130
		& 33.49 & 106 & 106 \\
		& Prox-ML-PnP
		& \cellcolor{green!20}33.31 & 239 & 145
		& \cellcolor{green!20}32.65 & 189 & 123
		& \cellcolor{green!20}31.11 & 165 & 99
		& \cellcolor{green!20}28.76 & 128 & 82 \\
		& ML-SPnP
		& \cellcolor{blue!15}43.51 & 270 & 200
		& \cellcolor{blue!15}34.02 & 180 & 130
		& 34.97 & 158 & 116
		& \cellcolor{blue!15}31.14 & 121 & 91 \\
		
		\bottomrule[1.2pt]
	\end{tabular}%
	}
	\caption{
		Computational cost for sparse-view \gls{CT} with $n_{\theta} \in \{20, 30, 40, 50\}$ projection views on the Lymph Nodes and CQ500 datasets. Entries report the mean reconstruction time in seconds per slice, \gls{NFE} per slice, and executed outer iterations or diffusion steps $K^\ast$. For each dataset and view count, the best and second-best mean times across all methods except \gls{FBP} are shown in bold and underlined, respectively. \colorbox{green!20}{Green} and \colorbox{blue!15}{blue} highlights indicate the best and second-best mean times, respectively, among the optimization-based methods (Prox-SNORE, ann-Prox-SNORE, Prox-ML-PnP, and ML-SPnP).
	}
	\label{tab:computational_cost_datasets}
\end{table*}

\subsubsection{Quantitative comparison}

Table~\ref{tab:quantitative_comparison_template} reports reconstruction quality on the Lymph Nodes and CQ500 datasets for $n_\theta \in \{20,30,40,50\}$. Both diffusion-based methods improve all three metrics relative to \gls{FBP}, although the \gls{PnP} methods achieve higher mean \gls{PSNR} and \gls{SSIM} and lower mean \gls{LPIPS}. ML-SPnP produces results close to annealed Prox-\gls{SNORE}, with differences in mean \gls{PSNR} and \gls{SSIM} not exceeding $0.26$~dB and $0.002$, respectively. These two methods exhibit overlapping $95\%$ confidence intervals and generally comparable patient-level \gls{SD}. At the selected iteration budgets, ML-SPnP achieves higher mean \gls{PSNR} than standard Prox-\gls{SNORE}, with gains ranging from $0.11$ to $1.63$~dB. However, Prox-\gls{ML}-\gls{PnP} achieves higher mean \gls{PSNR} and \gls{SSIM} in every setting, exceeding ML-SPnP by $0.39$--$0.90$~dB in \gls{PSNR}. ML-SPnP instead yields lower mean \gls{LPIPS}, with reductions of approximately $14$--$25\%$ relative to Prox-\gls{ML}-\gls{PnP}. The comparison therefore depends on the reconstruction metric, with no uniform quality advantage for ML-SPnP.

Table~\ref{tab:computational_cost_datasets} reports the corresponding computational costs in seconds, \gls{NFE} \cite{lu2022dpm} and iteration budget $K^\star$. \Gls{FBP} has the smallest reported runtime, at the expense of reconstruction quality. Both diffusion-based methods are also faster than ML-SPnP except at $40$ and $50$ views on Lymph Nodes, although their reconstruction scores are less favorable across all three metrics. The main benefit of ML-SPnP is its reduced runtime relative to standard Prox-\gls{SNORE}, with reductions of approximately $50$--$66\%$ on Lymph Nodes and $55$--$71\%$ on CQ500. The additional benefit over annealing is more limited: runtime reductions range from $5$--$31\%$ on Lymph Nodes and $7$--$35\%$ on CQ500 for $20$, $30$, and $50$ views, while no acceleration is observed at $40$ views on CQ500 ($34.97$~seconds versus $34.88$~seconds per slice). Prox-\gls{ML}-\gls{PnP} is faster than ML-SPnP in every setting, with ML-SPnP requiring approximately $2$--$31\%$ more time. Relative to this baseline, the lower \gls{LPIPS} of ML-SPnP thus comes with longer runtimes and lower \gls{PSNR}/\gls{SSIM}.

\subsubsection{Qualitative comparison}

Figures~\ref{fig:svct_20_40_qualitative} and~\ref{fig:svct_cq500_qualitative} show reconstructions on the Lymph Nodes and CQ500 datasets, respectively, for $n_\theta\in\{20,30,40,50\}$. \Gls{FBP} exhibits pronounced streak artifacts that obscure anatomical detail, particularly at extremely sparse-view settings. \Gls{DiffPIR} and \gls{DDS} substantially suppress these artifacts, but discrepancies remain in the magnified regions and corresponding error maps. \Gls{DiffPIR} improves visibly with additional views, while \gls{DDS} retains smoother internal kidney structure in the Lymph Nodes example and substantial errors around high-contrast boundaries in CQ500, including at $50$ views.
\noindent ML-SPnP produces reconstructions visually close to those of standard and annealed Prox-\gls{SNORE}, with similar local structures and residual-error patterns. The differences between these stochastic methods are generally modest in the illustrated regions. All retain some smoothing and loss of fine detail relative to the reference, particularly at $20$ views. Although additional projection views improve the visibility of local structures, residual errors remain for ML-SPnP at higher view counts.

Prox-\gls{ML}-\gls{PnP} appears slightly smoother and less sharp than ML-SPnP and the standard and annealed Prox-\gls{SNORE} variants. This difference is more pronounced in the $20$-view setting, where fine details and local intensity variations are more strongly smoothed. The stochastic methods retain more visible local texture, although they also exhibit residual smoothing relative to the reference. This observation is consistent with the lower \gls{LPIPS} of ML-SPnP, despite the higher \gls{PSNR} and \gls{SSIM} achieved by Prox-\gls{ML}-\gls{PnP}.

\subsubsection{Variability across random seeds}\label{sec:stochastic_variability}

To assess sensitivity to the random seed, we performed $10$ independent ML-SPnP reconstructions for each acquisition setting, with $n_\theta\in\{20,30,40,50\}$. Figure~\ref{fig:stochastic_variability} shows the mean reconstruction, its \gls{AE} relative to the reference, and the pixelwise \gls{SD} across runs for a Lymph Nodes slice. The variability is spatially heterogeneous, with larger deviations around anatomical boundaries and fine structures. For this slice, both the reconstruction error and the variability across seeds generally decrease as the number of views increases, although localized errors and variability remain at $50$ views. These observations suggest improved repeatability with additional projection data.
\begin{figure}[H]
	\centering
	\begingroup
	\footnotesize
	\setlength{\tabcolsep}{0pt}
	\newcommand{\uncgap}{0.5pt} 
	\newcommand{\uncbargap}{1pt} 
  \newcommand{\uncbarwidth}{4pt} 
  \newcommand{\uncnumberscale}{0.65} 
  \newcommand{\uncnumber}[1]{%
    \scalebox{\uncnumberscale}{\tiny
    	\strut #1}}
  \newcommand{\uncpanel}[1]{%
    \raisebox{-0.5\height}{%
      \includegraphics[width=.20\linewidth]{%
        figures/experiments/uncertainity_bis/#1}}}
  \newsavebox{\uncrefbox}
  \newsavebox{\unclabelbox}
  \newlength{\uncpanelheight}
  \newlength{\unclabelheight}
  \newlength{\unclabelwidth}
  \newlength{\uncbarheight}
  \sbox{\uncrefbox}{%
    \includegraphics[width=.20\linewidth]{%
      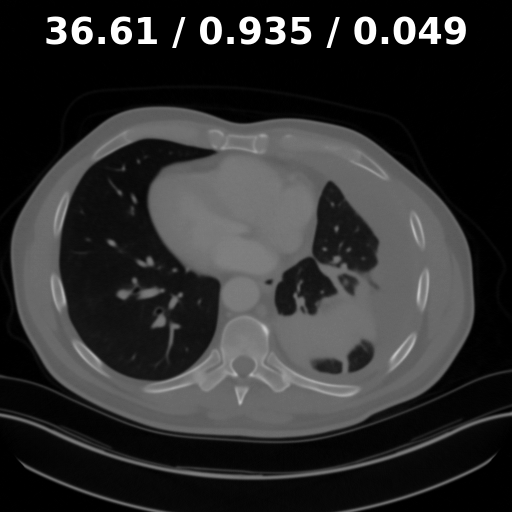}}
  \setlength{\uncpanelheight}{%
    \dimexpr\ht\uncrefbox+\dp\uncrefbox\relax}
  \sbox{\unclabelbox}{\uncnumber{200}}
  \setlength{\unclabelheight}{%
    \dimexpr\ht\unclabelbox+\dp\unclabelbox\relax}
  \setlength{\unclabelwidth}{%
    \dimexpr\wd\unclabelbox+2pt\relax}
  \setlength{\uncbarheight}{%
    \dimexpr\uncpanelheight-2\unclabelheight-4pt\relax}
  \newcommand{\uncbar}[3]{%
    \raisebox{-0.5\height}{%
      \vbox to \uncpanelheight{%
        \offinterlineskip
        \hbox to \unclabelwidth{%
          \hfil\uncnumber{#3}\hfil}%
        \kern2pt
        \hbox to \unclabelwidth{%
          \hfil
          \includegraphics[
            width=\uncbarwidth,
            height=\uncbarheight
          ]{figures/experiments/uncertainity_bis/#1}%
          \hfil}%
        \kern2pt
        \hbox to \unclabelwidth{%
          \hfil\uncnumber{#2}\hfil}%
        \kern0pt
      }}}
  \resizebox{\columnwidth}{!}{%
    \begin{tabular}{@{}c@{\hspace{3pt}}
      c@{\hspace{\uncgap}}
      c@{\hspace{\uncgap}}
      c@{\hspace{\uncgap}}
      c@{\hspace{\uncbargap}}
      c@{}}
      \raisebox{-0.5\height}{\rotatebox{90}{\textbf{Mean}}}
      & \uncpanel{20_views_mean.png}
      & \uncpanel{30_views_mean.png}
      & \uncpanel{40_views_mean.png}
      & \uncpanel{50_views_mean.png}
      &
      \\ \noalign{\vskip2pt}
      \raisebox{-0.5\height}{\rotatebox{90}{\textbf{AE}}}
      & \uncpanel{20_views_absolute_error.png}
      & \uncpanel{30_views_absolute_error.png}
      & \uncpanel{40_views_absolute_error.png}
      & \uncpanel{50_views_absolute_error.png}
      & \uncbar{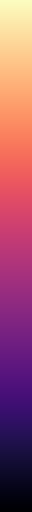}{0}{150}
      \\ \noalign{\vskip2pt}
      \raisebox{-0.5\height}{\rotatebox{90}{\textbf{SD}}}
      & \uncpanel{20_views_std.png}
      & \uncpanel{30_views_std.png}
      & \uncpanel{40_views_std.png}
      & \uncpanel{50_views_std.png}
      & \uncbar{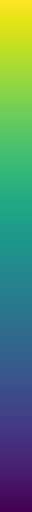}{0}{150}
      \\ \noalign{\vskip2pt}
      & \textbf{20 views}
      & \textbf{30 views}
      & \textbf{40 views}
      & \textbf{50 views}
      &
    \end{tabular}%
  }
  \endgroup

\caption{ML-SPnP reconstructions of a Lymph Nodes slice with $n_\theta\in\{20,30,40,50\}$ views, showing the mean and pixelwise \gls{SD} over $10$ independent runs and the \gls{AE} between the mean and reference. Display ranges are $[-1000,1500]$~\gls{HU} for the mean and $[0,150]$~\gls{HU} for \gls{SD} and \gls{AE}.}
  \label{fig:stochastic_variability}
\end{figure}


\begin{figure*}[!t]
    \centering

    \includegraphics[width=\linewidth]
        {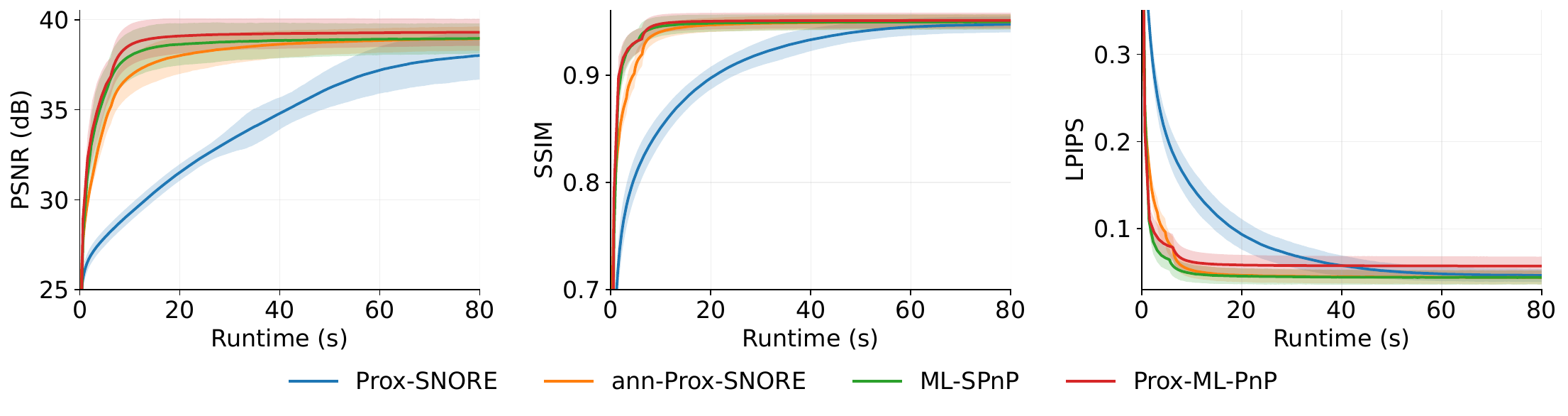}
    \caption{
        Patient-level mean \gls{PSNR}, \gls{SSIM}, and \gls{LPIPS}
        versus runtime in seconds for reconstruction on Lymph Nodes using $n_{\theta}=20$ projection views. Shaded bands indicate
        $\pm$ one patient-level \gls{SD}.
    }
    \label{fig:lymph-ct20-metric-curves}

    \medskip

    \includegraphics[width=\linewidth]
        {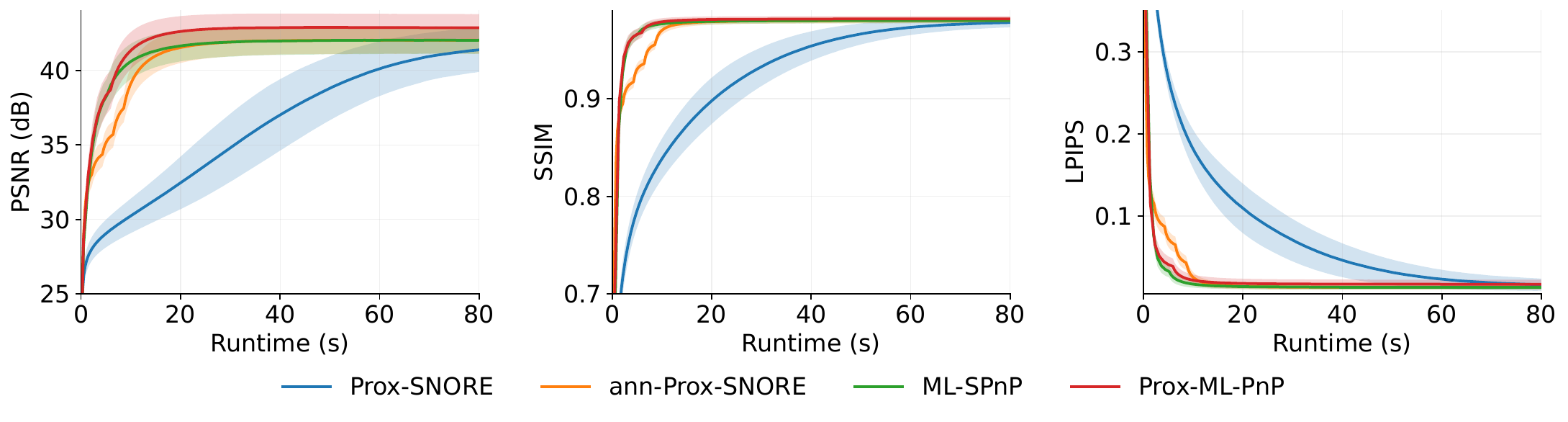}
    \caption{
        Patient-level mean \gls{PSNR}, \gls{SSIM}, and \gls{LPIPS}
        versus runtime in seconds for reconstruction on CQ500 using $n_{\theta}=40$ projection views. Shaded bands indicate
        $\pm$ one patient-level \gls{SD}.
    }
    \label{fig:cq500-ct40-metric-curves}
\end{figure*}

\subsubsection{Convergence and runtime analysis}

Figures~\ref{fig:lymph-ct20-metric-curves} and~\ref{fig:cq500-ct40-metric-curves} show the evolution of patient-level mean \gls{PSNR}, \gls{SSIM}, and \gls{LPIPS} with elapsed reconstruction time for Lymph Nodes with $n_{\theta}=20$ views and CQ500 with $n_{\theta}=40$ views, respectively. The shaded bands represent patient-level variability. These trajectories complement Tables~\ref{tab:computational_cost_datasets} and~\ref{tab:quantitative_comparison_template} by showing how reconstruction quality evolves before and beyond the selected iteration budgets.

In the $20$-view Lymph Nodes setting, ML-SPnP exhibits faster initial improvements than both standard and annealed Prox-\gls{SNORE}. Standard Prox-\gls{SNORE} improves more gradually across all three metrics, requiring substantially more time to approach comparable reconstruction quality. Relative to the annealed variant, ML-SPnP maintains a \gls{PSNR} advantage during much of the initial and intermediate phases, although both methods approach similar plateau values. Prox-\gls{ML}-\gls{PnP} also improves rapidly and reaches higher \gls{PSNR} and slightly higher \gls{SSIM} plateaus than ML-SPnP, but its \gls{LPIPS} remains higher.

In the $40$-view CQ500 setting, ML-SPnP again improves more rapidly than standard Prox-\gls{SNORE}, but its advantage over the annealed variant is mainly confined to the initial phase. Their curves nearly coincide after approximately $30$~seconds, and their selected runtimes are essentially identical. The \gls{ML} strategy therefore provides little additional benefit over annealing once this initial phase has passed. As on Lymph Nodes, Prox-\gls{ML}-\gls{PnP} reaches higher \gls{PSNR} and slightly higher \gls{SSIM} plateaus, while ML-SPnP achieves lower \gls{LPIPS}. However, this \gls{LPIPS} difference is smaller than in the $20$-view Lymph Nodes example.

\subsection{Ablation Study}\label{sec:ablation}

\subsubsection{Effect of the coherence terms}\label{sec:cost_coherence}

We evaluated the practical effect of the coherence terms on reconstruction quality and runtime. Figure~\ref{fig:coherence-runtime} compares three variants: removing the entire coherence correction, retaining only the data-fidelity correction $\boldc_{f,\ell}^{(k)}$, and retaining the full correction $\boldc_{f,\ell}^{(k)}+\lambda\boldc_{g,\ell}^{(k)}$. The second variant corresponds to the proposed ML-SPnP formulation.

\definecolor{coherenceNone}{HTML}{747474}
\definecolor{coherenceData}{HTML}{43106B}
\definecolor{coherenceFull}{HTML}{A68BC7}

\pgfplotsset{
  compat=1.18,
  coherence/none/.style={
    color=coherenceNone,
    line width=0.9pt,
    dashed,
    no marks
  },
  coherence/data/.style={
    color=coherenceData,
    line width=1pt,
    mark=*,
    mark repeat=10,
    mark size=1.4pt,
    mark options={solid,fill=coherenceData}
  },
  coherence/full/.style={
    color=coherenceFull,
    line width=1pt,
    densely dashed,
    mark=square*,
    mark repeat=10,
    mark size=1.4pt,
    mark options={solid,fill=white}
  }
}

\begin{figure}[H]
	
	\centering
	
	\pgfplotstableread{
	iteration runtime psnr
	0 0 16.37337112
	1 0.6041986719 24.97376251
	2 1.403493173 26.65750885
	3 2.203866176 27.3802948
	4 3.001679156 27.56430435
	5 3.800894987 27.51992035
	6 4.598032124 27.55895996
	7 5.421894126 27.52101135
	8 5.613775685 27.64348602
	9 5.80525123 27.7376709
	10 5.996718995 27.8239212
	11 6.188246734 27.91208649
	12 6.379826576 27.97562408
	13 6.571404867 28.04800034
	14 6.762833614 28.11779022
	15 6.953591522 28.18321228
	16 7.145322703 28.22455597
	17 7.336330112 28.2783699
	18 7.527644999 28.34300232
	19 7.71869136 28.39916611
	20 7.909994833 28.43179321
	21 8.101203963 28.47090149
	22 8.293134134 28.53980255
	23 8.485011093 28.61665726
	24 8.67669655 28.64206696
	25 8.868056942 28.68338394
	26 9.059406787 28.74248886
	27 9.25115386 28.80591202
	28 9.442410454 28.84542847
	29 9.63377719 28.88326645
	30 9.825193457 28.9245491
	31 10.01695942 28.9714241
	32 10.20881336 29.01283646
	33 10.40010492 29.05787659
	34 10.59190728 29.11445618
	35 10.78272418 29.11284256
	36 10.97462826 29.14537811
	37 11.16597952 29.17839813
	38 11.35763301 29.21043396
	39 11.54992058 29.2431488
	40 11.74157607 29.26235199
	41 11.93304038 29.29810333
	42 12.12431936 29.31245422
	43 12.31571007 29.34256744
	44 12.50715025 29.3683548
	45 12.69877687 29.40737915
	46 12.89068034 29.43075943
	47 13.08199672 29.46575546
	48 13.27341591 29.52600861
	49 13.46498983 29.54645157
	50 13.65701181 29.58589935
	}\noneData
	
	\pgfplotstableread{
	iteration runtime psnr
	0 0 16.37282181
	1 0.6175931022 24.96947479
	2 1.421304759 27.78807449
	3 2.224930689 29.4849472
	4 3.028378367 30.62223053
	5 3.833141487 31.59819031
	6 4.638255019 32.40119934
	7 5.446800195 32.89276886
	8 5.636930913 33.09656143
	9 5.826802295 33.25239182
	10 6.017386958 33.43567276
	11 6.206949867 33.55838013
	12 6.397391986 33.69180679
	13 6.587777421 33.80750656
	14 6.778392497 33.90764236
	15 6.968145203 34.0038681
	16 7.15894068 34.08328247
	17 7.348584879 34.1634407
	18 7.538834825 34.24237823
	19 7.728484157 34.3171196
	20 7.919067368 34.36039352
	21 8.110591076 34.41353607
	22 8.301093999 34.48469925
	23 8.491255458 34.55198669
	24 8.681582175 34.60107803
	25 8.872457817 34.61824036
	26 9.062564846 34.70084
	27 9.253197584 34.71701813
	28 9.443484653 34.776577
	29 9.633840788 34.8180542
	30 9.824417535 34.82323074
	31 10.0145573 34.83000565
	32 10.20536346 34.85267639
	33 10.39602308 34.84197998
	34 10.58612034 34.90279007
	35 10.77684759 34.89938736
	36 10.96715467 34.94394302
	37 11.15758289 34.93209839
	38 11.34853956 34.94559097
	39 11.54043867 34.96272278
	40 11.73181295 34.97751999
	41 11.92322103 34.98215485
	42 12.11435938 34.99315262
	43 12.30564209 35.00082779
	44 12.49693957 35.06653595
	45 12.68822174 35.1212616
	46 12.88021718 35.1312027
	47 13.071533 35.17528915
	48 13.26265529 35.19987869
	49 13.45376716 35.21360779
	50 13.64497787 35.2486496
	}\dataData
	
	\pgfplotstableread{
	iteration runtime psnr
	0 0 16.37268829
	1 0.6132737398 24.9341011
	2 1.532008555 27.83472061
	3 2.453348391 29.54385376
	4 3.376279335 30.50886917
	5 4.302945856 31.29671478
	6 5.226337194 32.24195099
	7 6.150903907 32.84558868
	8 6.339899566 33.02190399
	9 6.52906711 33.16881561
	10 6.718063127 33.3444519
	11 6.906928796 33.44683456
	12 7.095628649 33.57250977
	13 7.285098724 33.6618576
	14 7.474113386 33.76596832
	15 7.663592696 33.8703804
	16 7.853087243 33.96064758
	17 8.04240261 34.0358429
	18 8.231123164 34.12982941
	19 8.42006233 34.18546677
	20 8.609394241 34.27179337
	21 8.797977138 34.36002731
	22 8.987264536 34.41490936
	23 9.176547889 34.49281693
	24 9.365821455 34.57197189
	25 9.55483932 34.6111145
	26 9.744890746 34.68244171
	27 9.933708247 34.7146759
	28 10.12297788 34.75492096
	29 10.31227947 34.79920197
	30 10.50146094 34.84828949
	31 10.69138942 34.87956238
	32 10.88039831 34.94626999
	33 11.06967709 34.97599411
	34 11.25910719 35.01432037
	35 11.4483125 35.05081177
	36 11.63715405 35.06137085
	37 11.82653382 35.08075714
	38 12.01621114 35.12878418
	39 12.20616191 35.11564255
	40 12.39562929 35.14736938
	41 12.58548534 35.20339203
	42 12.7747812 35.22079468
	43 12.9643134 35.21332169
	44 13.15302333 35.19573593
	45 13.34280313 35.22106171
	46 13.53175739 35.23709106
	47 13.72151006 35.24607468
	48 13.91095175 35.26683044
	49 14.10098091 35.28324127
	50 14.29047757 35.31109238
	}\fullData
	
	\begin{tikzpicture}[font=\footnotesize\sffamily]
		\begin{axis}[
		  width=0.95\linewidth,
		  height=5.5cm,
		  xlabel={Runtime (s)},
		  ylabel={PSNR (dB)},
		  xmin=0,
		  xmax=14.6,
		  ymin=16,
		  ymax=36,
		  xtick={0,4,8,12},
		  ytick={16,20,24,28,32,36},
		  grid=major,
		  grid style={gray!15},
		  axis line style={gray!65},
		  tick style={gray!65},
		  tick align=outside,
		  scaled ticks=false,
		  tick label style={font=\footnotesize
		  },
		  label style={font=\footnotesize
		  },
		  title style={font=\small
		  },
		  legend columns=4,
		  legend cell align=left,
		  legend style={
		    at={(0.5,-0.25)},
		    anchor=north,
		    draw=none,
		    fill=none,
		    font=\footnotesize,
		    /tikz/every even column/.append style={column sep=8pt}
		  }
		]
		
			\addplot[coherence/none,forget plot]
			  table[x=runtime,y=psnr]{\noneData};
			
			\addplot[coherence/data,forget plot]
			  table[x=runtime,y=psnr]{\dataData};
			
			\addplot[coherence/full,forget plot]
			  table[x=runtime,y=psnr]{\fullData};
			
			\addlegendimage{empty legend}
			\addlegendentry{Coherence term:}
			
			\addlegendimage{coherence/none}
			\addlegendentry{$\varnothing$}
			
			\addlegendimage{coherence/data,mark repeat=1}
			\addlegendentry[
			  text height={height("$c$")},
			  text depth={depth("$c$")}
			]{$c_{f,\ell}^{(k)}$}
			
			\addlegendimage{coherence/full,mark repeat=1}
			\addlegendentry[
			  text height={height("$c$")},
			  text depth={depth("$c$")}
			]{$\boldc_{f,\ell}^{(k)}+\lambda \boldc_{g,\ell}^{(k)}$}
			
		\end{axis}
	\end{tikzpicture}
	
	\caption{
		\Gls{PSNR} versus runtime in seconds for three coherence variants on a Lymph Nodes slice reconstructed from $n_\theta=20$ projection views. Markers indicate every $10$ outer iterations along the corrected trajectories.
		}
	\label{fig:coherence-runtime}
\end{figure}
For the illustrated 20-view reconstruction, removing the entire correction substantially slows the improvement in \gls{PSNR}. At outer iteration $k = 50$, this variant reaches $29.59$~dB, compared with $35.25$~dB when retaining only data-fidelity coherence and $35.31$~dB with the full correction. The two corrected variants achieve similar \gls{PSNR} values at corresponding outer iterations, as indicated by the nearly aligned markers along the vertical axis. Omitting the prior correction $\boldc_{g,\ell}^{(k)}$ avoids the additional denoiser evaluations required to estimate it, reducing runtime by approximately $4.5\%$ at $k=50$. These results support retaining the data-fidelity correction while omitting the prior correction in the configuration considered.
\input{figures/ablation/ml_steps}

\subsubsection{Effect of the  \gls{MR} and  \gls{ML} steps}

Figure~\ref{fig:mlspnp_combined_ablation} investigates the respective effects of \gls{MR} initialization and subsequent \gls{ML} corrections. The left and middle panels compare two- and three-level configurations with $K_\ell=5$ inner iterations and $K_{\mathrm{ML}}\in\{0,1,2,3,4,5,6,7,8\}$. All these configurations include \gls{MR} initialization; the case $K_{\mathrm{ML}}=0$ applies only finest-level updates thereafter. The right panel compares reconstructions with and without \gls{MR} initialization for $L=3$ and $K_{\mathrm{ML}}=6$, using the same number of inner iterations. The first $10$~seconds of reconstruction are shown to assess the early improvement in \gls{PSNR} relative to computational cost.

The \gls{MR} step provides a substantial improvement in the initial reconstruction. In the right panel, it increases \gls{PSNR} from approximately $16.37$~dB for the \gls{FBP} starting point to $24.87$~dB in about $0.61$~seconds. Although omitting this step avoids its initial computational cost, the subsequent \gls{ML} corrections do not compensate for the poorer starting estimate within the displayed interval. At $10$~seconds, the configuration with \gls{MR} initialization reaches $34.92$~dB, compared with $29.67$~dB without it, corresponding to a gain of approximately $5.25$~dB. This comparison supports the role of \gls{MR} initialization as an effective warm start for the subsequent reconstruction.

The left and middle panels show that \gls{ML} corrections provide further acceleration beyond \gls{MR} initialization. With three levels, the configuration with $K_{\mathrm{ML}}=0$ reaches approximately $28.2$~dB at $10$~seconds, whereas using six \gls{ML} corrections raises \gls{PSNR} to about $34.9$~dB. The number of resolution levels also affects this gain: with $K_{\mathrm{ML}}=6$, the two-level configuration reaches about $31.9$~dB over the same interval. Thus, the additional coarse level improves the accuracy--runtime compromise for this configuration.

However, increasing the number of \gls{ML} corrections does not necessarily provide a further benefit under a fixed runtime budget. In the three-level setting, increasing $K_{\mathrm{ML}}$ from $6$ to $8$ yields about $34.7$~dB at $10$~seconds, slightly below the value obtained with six corrections. Extending the \gls{ML} phase increases its computational cost and leaves less time for subsequent finest-level updates. For the settings considered, these results support combining \gls{MR} initialization with a limited number of \gls{ML} corrections to accelerate the early reconstruction.

	\section{Discussion}\label{sec:discussion}

The proposed ML-SPnP method introduces a \gls{ML} strategy for accelerating stochastic \gls{PnP} reconstruction. However, several limitations remain.

First, the proposed \gls{ML} phase should be interpreted as a finite acceleration or warm-start procedure rather than as a mechanism for which we establish an independent convergence guarantee. In particular, we do not claim that the \gls{ML} steps themselves induce convergence; any asymptotic convergence properties are inherited from the finest-level stochastic \gls{PnP} algorithm used after the \gls{ML} phase. 

Second, compared with Prox-\gls{ML}-\gls{PnP}, ML-SPnP achieves lower mean \gls{LPIPS} but lower mean \gls{PSNR} and \gls{SSIM} and longer runtimes in the evaluated settings. These results indicate a trade-off between the evaluated reconstruction metrics and computational cost. Although the qualitative examples and \gls{LPIPS} scores suggest differences in perceptual similarity to the reference images, they do not establish better preservation of diagnostically relevant structures or greater clinical safety. Such claims would require task-based clinical assessment.

Third, although the experiments presented here are limited to \gls{2D} reconstruction, the proposed \gls{ML} formulation naturally extends to \gls{3D} using the corresponding projection, backprojection, restriction, and prolongation operators. The higher computational cost of volumetric reconstruction offers substantial scope for acceleration through coarse-level updates. To avoid training fully volumetric denoisers, one could combine the appropriate data-consistency operator with the existing networks applied slice-wise, while taking additional care to maintain inter-slice consistency. Adapting the noise coordination across slices proposed for diffusion-based reconstruction \cite{du2026improving} could help address this issue without an explicit inter-slice \gls{TV} penalty.

Finally, the experiments are limited to simulated sparse-view acquisitions. Further validation on real clinical data would help assess robustness under practical acquisition conditions.
	\section{Conclusion}\label{sec:conclusion}

We introduced ML-SPnP, a \gls{ML} stochastic \gls{PnP} method for sparse-view \gls{CT} reconstruction formulated on wavelet \gls{MR} spaces. The method combines coarse-to-fine initialization with a finite sequence of \gls{ML} corrections using resolution-specific denoisers.

Experiments on the Lymph Nodes and CQ500 datasets show that ML-SPnP reaches reconstruction quality comparable to the evaluated stochastic \gls{PnP} baselines, while reducing runtime across sparse-view settings. Overall, these findings indicate that \gls{ML} stochastic \gls{PnP} is a promising direction for efficient and reliable sparse-view \gls{CT} reconstruction. Future work will investigate extensions to volumetric reconstruction and validation on real acquisition data.

	
	\appendix
	
	\AtNextBibliography{\scriptsize} \printbibliography

\end{document}